\documentclass{article}

\usepackage{subfig}
\usepackage{amsmath}
\usepackage{amsthm}
\usepackage{amssymb}
\usepackage{graphicx}

\numberwithin{equation}{section}
\theoremstyle{plain}
\newtheorem{thm}{\protect\theoremname}
\theoremstyle{plain}
\newtheorem{prop}[thm]{\protect\propositionname}
\theoremstyle{definition}
\newtheorem{defn}[thm]{\protect\definitionname}
\theoremstyle{remark}
\newtheorem{rem}[thm]{\protect\remarkname}
\theoremstyle{plain}
\newtheorem{cor}[thm]{\protect\corollaryname}
\theoremstyle{plain}
\newtheorem{lem}[thm]{\protect\lemmaname}

\usepackage{babel}
\providecommand{\corollaryname}{Corollary}
\providecommand{\definitionname}{Definition}
\providecommand{\lemmaname}{Lemma}
\providecommand{\propositionname}{Proposition}
\providecommand{\remarkname}{Remark}
\providecommand{\theoremname}{Theorem}

     \usepackage[final,nonatbib]{neurips_2022}

\usepackage[utf8]{inputenc} 
\usepackage[T1]{fontenc}    
\usepackage{hyperref}       
\usepackage{url}            
\usepackage{booktabs}       
\usepackage{amsfonts}       
\usepackage{nicefrac}       
\usepackage{microtype}      
\usepackage{xcolor}         

\author{%
  Arthur Jacot\\
  Courant Institute of Mathematical Sciences\\
  New York University\\
  \texttt{arthur.jacot@nyu.edu} \\
  \And
  Eugene Golikov\\
  Chair of Statistical Field Theory\\
  École Polytechnique Fédérale de Lausanne\\
  \texttt{evgenii.golikov@epfl.ch} \\
  \And
  Clément Hongler\\
  Chair of Statistical Field Theory\\
  École Polytechnique Fédérale de Lausanne\\
  \texttt{clement.hongler@epfl.ch} \\
  \And
  Franck Gabriel\\
  Institut de Science Financière et d'Assurances\\
  Université Lyon 1\\
  \texttt{franckr.gabriel@gmail.com} \\
}

\begin{document}
\title{Feature Learning in $L_{2}$-regularized DNNs:\\
Attraction/Repulsion and Sparsity}
\maketitle
\begin{abstract}
We study the loss surface of DNNs with $L_{2}$ regularization. We
show that the loss in terms of the parameters can be reformulated
into a loss in terms of the layerwise activations $Z_{\ell}$ of the
training set. This reformulation reveals the dynamics behind feature
learning: each hidden representations $Z_{\ell}$ are optimal w.r.t.
to an attraction/repulsion problem and interpolate between the input
and output representations, keeping as little information from the
input as necessary to construct the activation of the next layer.
For positively homogeneous non-linearities, the loss can be further
reformulated in terms of the covariances of the hidden representations,
which takes the form of a partially convex optimization over a convex
cone.

This second reformulation allows us to prove a sparsity result for
homogeneous DNNs: any local minimum of the $L_{2}$-regularized loss
can be achieved with at most $N(N+1)$ neurons in each hidden layer
(where $N$ is the size of the training set). We show that this bound
is tight by giving an example of a local minimum that requires $N^{2}/4$
hidden neurons. But we also observe numerically that in more traditional
settings much less than $N^{2}$ neurons are required to reach the
minima.
\end{abstract}

\section{Introduction}

It is generally believed that the success of deep learning hinges
on the ability of deep neural networks (DNNs) to learn features that
are well suited to the task they are trained on. There is however
little understanding of what these features are and how they are selected by the network.

On the other hand, recent results \cite{jacot2018neural} have shown
that it is possible to train DNNs without feature learning. This suggests
the existence of two regimes of DNNs, a kernel regime (also called
lazy or NTK regime) without feature learning and an active regime
where features are learned. The presence or absence of feature learning
can depend on multiple factors, such as the initialization/parametrization
of DNNs \cite{chizat2018note,yang2020feature,li2020towards,jacot-2021-DLN-Saddle},
very large depths \cite{hanin2019finite} or large learning rate \cite{lewkowycz_2020_large_lr,cohen_2021_edge_of_stability}.

In this paper, we focus on the impact of $L_{2}$ regularization on
feature learning in DNNs. This analysis is further motivated by recent results \cite{gunasekar_2018_implicit_bias,gunasekar_2018_implicit_bias2,chizat_2020_implicit_bias} which show that the implicit bias of gradient descent on losses such as the cross entropy (which decay exponentially towards infinity) is essentially the same as the bias induced by $L_2$ regularization in DNNs.

Generally, the bias induced by the addition of $L_{2}$-regularization
on the parameters $\theta$ of a model $f_{\theta}$ can be described
by the \emph{representation cost} $R(f)=\min_{\theta:f_{\theta}=f}\left\Vert \theta\right\Vert ^{2}$,
since $\min_{\theta}C(f_{\theta})+\lambda\left\Vert \theta\right\Vert ^{2}=\min_{f}C(f)+\lambda R(f)$.

In deep linear networks, the addition of $L_{2}$ regularization on
the parameters corresponds to the addition of an $L_{p}$-Schatten
norm regularization to the represented matrix, with $p=\nicefrac{2}{L}$
where $L$ is the depth of the network \cite{gunasekar_2018_implicit_bias2,dai_2021_repres_cost_DLN}.
This implies a sparsity effect that increases with depth $L$.

In non-linear networks the sparsity effect of $L_{2}$-regularization
has been described for shallow networks ($L=2$) in \cite{bach2017_F1_norm, Ongie_2020_repres_bounded_norm_shallow_ReLU_net,Savarese_2019_repres_bounded_norm_shallow_ReLU_net_1D} or for shallow non-linear networks with added linear layers \cite{ongie2022_linear_layer_in_DNN}.
Though it seems natural that this effect should become stronger for
deeper networks, to our knowledge little theoretical work has been done in this area.

\subsection{Contributions}

In this paper, we study the minima of the loss of $L_{2}$ regularized
fully-connected DNNs of depth $L$. We propose two reformulations
of the loss:
\begin{enumerate}
\item The first reformulation expresses the loss in terms of the representations
$Z_{1},\dots,Z_{L}$ (the layer pre-activations of every input in
the training set) of every layer of the network. This reformulation
has the advantage of being local - the optimal choice of a layer $Z_{\ell}$
only depends on its neighboring layers $Z_{\ell-1}$ and $Z_{\ell+1}$.
The optimal choice of representation $Z_{\ell}$ is at the balance
between an attractive force (determined by the previous layer) and
a repulsive force (coming from the next layer). It illustrates how
the representations $Z_{1},\dots,Z_{L-1}$ interpolates between the
input layer $Z_{0}$ and output layer $Z_{L}$.
\item The second reformulation expresses the loss in terms of the covariances
of the representation before applying the non-linearity $K_{\ell}=Z_{\ell}^{T}Z_{\ell}$
and after the non-linearity $K_{\ell}^{\sigma}=\left(Z_{\ell}^{\sigma}\right)^{T}Z_{\ell}^{\sigma}$.
For positively homogeneous non-linearities and when the number of
neurons $n_{\ell}$ in every hidden layer $\ell$ is larger or equal
to $N(N+1)$ for $N$ the number of datapoints, this reformulation
is an optimization of a (partially convex) loss over the covariances
$(K_{1},K_{1}^{\sigma}),\dots,(K_{L-1},K_{L-1}^{\sigma})$ and the
outputs $Z_{L}$, restricted to a (translated) convex cone. This
reformulation does not depend on the number of neurons $n_{\ell}$
in the hidden layers (as long as $n_{\ell}\geq N(N+1)$).
\end{enumerate}
The second reformulation implies that for positively homogeneous non-linearities
such as the ReLU, as the number of neurons in the hidden layers $n_{\ell}$
increase, the global minimum of the $L_{2}$-regularized loss goes
down until reaching a plateau (i.e. adding neurons does not lead to
an improvement in the loss). This illustrates the sparsity effect of
$L_{2}$-regularization, where the optimum reached on a very large
network is equivalent to a much smaller network. 

The start of the plateau hence gives a measure of sparsity of the
global minimum. We show that the minimal number of neurons $n_{\ell}$
to reach this plateau is determined by a notion of rank $\mathrm{Rank}_{\sigma}\left(K_{\ell},K_{\ell}^{\sigma}\right)$
of the covariance pairs. We show that $\mathrm{Rank}_{\sigma}\left(K_{\ell},K_{\ell}^{\sigma}\right)\leq N(N+1)$,
i.e. the plateau must start before $N(N+1)$ and that the scaling
of this upper bound is tight by giving an example dataset such that
at the optimum $\mathrm{Rank}_{\sigma}\left(K_{\ell},K_{\ell}^{\sigma}\right)\geq N^{2}/4$.
We also present other datasets where the start of the plateau is either constant or grows linearly with the number of datapoints.
We also observe empirically that the plateau can start at much smaller widths
for real data such as MNIST and the teacher/student setting.

\section{Setup}

We consider fully-connected deep neural networks with $L+1$ layers,
numbered from $0$ (the input layer) to $L$ (the output layer), with
nonlinear activation function $\sigma:\mathbb{R}\to\mathbb{R}$ (e.g.
the ReLU $\sigma(x)=\max\{0,x\}$)). Each layer $\ell$ contains $n_{\ell}$
neurons and we denote $\mathbf{n}=\left(n_{1},\dots,n_{L}\right)$
the widths of the network. Given an input dataset $\{x_{1},\ldots,x_{N}\}\subset\mathbb{R}^{n_{0}}$
of size $N$, we consider the data matrix $X=(x_{0},\ldots,x_{n})\in\mathbb{R}^{n_{0}\times N}$,
and encode the activations and preactivations of the whole data set
by considering the pre-activations $Z_{\ell}(X;\mathbf{W})\in\mathbb{R}^{n_{\ell}\times N}$
and activations $Z_{\ell}^{\sigma}(X;\mathbf{W})\in\mathbb{R}^{(n_{\ell}+1)\times N}$
given by:
\begin{align*}
Z_{0}^{\sigma}(X;\mathbf{W}) & =\left(\begin{array}{c}
X\\
\beta\mathbf{1}_{N}^{T}
\end{array}\right)\\
Z_{\ell}(X;\mathbf{W}) & =W_{\ell}Z_{\ell-1}^{\sigma}(X;\mathbf{W})\\
Z_{\ell}^{\sigma}(X;\mathbf{W}) & =\left(\begin{array}{c}
\sigma\left(Z_{\ell}\right)\\
\beta\mathbf{1}_{N}^{T}
\end{array}\right),
\end{align*}
where $\mathbf{W}=\left(W_{\ell}\right)_{\ell=1,\dots,L}$ is the
collection of $n_{\ell}\times (n_{\ell-1}+1)$-dim weight matrices
$W_{\ell}$, $\sigma\left(Z_{\ell}\right)$ is obtained by applying
elementwise the nonlinearity $\sigma$ to the matrix $Z_{\ell}$,
and the scalar $\beta\in\mathbb{R}$ represents the amount of bias
(i.e. when $\beta=0$ there is no bias, when $\beta=1$ this definition
is equivalent to the traditional definition of bias). The output of
the network is the pre-activation of the $L$-th layer $Z_{L}$.

We often drop the dependence on the weights $\mathbf{W}$ and on the
dataset $X$ and simply write $Z_{\ell}$ and $Z_{\ell}^{\sigma}$.

We denote $f_{\mathbf{W}}:\mathbb{R}^{n_{0}}\to\mathbb{R}^{n_{L}}$
the \emph{network function}, which maps an input $x$ to the pre-activation
at the last layer.

\subsection{$L_{2}$-Regularized Loss and Representation Cost}

Given a general cost functional $C:\mathbb{R}^{n_{L}\times N}\to\mathbb{R}$,
the $L_{2}$-regularized loss of DNNs of widths $\mathbf{n}$ is
\[
\mathcal{L}_{\lambda,\mathbf{n}}(\mathbf{W})=C(Z_{L}(X;\mathbf{W}))+\lambda\left\Vert \mathbf{W}\right\Vert ^{2},
\]
where $\left\Vert \mathbf{W}\right\Vert $ is the $L_{2}$-norm of
$\mathbf{W}$ understood as a vector. Note that $\left\Vert \mathbf{W}\right\Vert ^{2}=\sum_{\ell=1}^{L}\left\Vert W_{\ell}\right\Vert _{F}^{2}$
where $\left\Vert \cdot\right\Vert _{F}$ denotes the Frobenius norm.
From now on, we often omit to specify the widths $\mathbf{n}$ and
simply write $\mathcal{L}_{\lambda}$.

The additional regularization cost should bias the network toward
low norm solutions. This bias on the parameters leads to a bias in
function space, which is described by the so-called \emph{representation
cost} $\mathcal{R}_{\mathbf{n}}(f)$ defined on functions $f:\mathbb{R}^{n_{0}}\to\mathbb{R}^{n_{L}}$:
\[
\mathcal{R}_{\mathbf{n}}(f)=\min_{\mathbf{W}:f_{\mathbf{W}}=f}\left\Vert \mathbf{W}\right\Vert ^{2},
\]
where the minimum is taken over all choices of parameters $\mathbf{W}$
of a width $\mathbf{n}$ network, with fixed $\beta$ bias amount,
such that the network function $f_{\mathbf{W}}$ equals $f$. By convention,
if no such parameters exist then $\mathcal{R}_{\mathbf{n}}(f)=+\infty$. 

Similarly, given an input-output pair $X\in\mathbb{R}^{n_{0}\times N}$,
$Y\in\mathbb{R}^{n_{L}\times N}$, the representation cost $R_{\mathbf{n}}(X,Y)$
is: 
\[
R_{\mathbf{n}}(X,Y)=\min_{\mathbf{W}:Z_{L}(X,\mathbf{W})=Y}\left\Vert \mathbf{W}\right\Vert ^{2},
\]
with again the convention that if there exists no weight $\mathbf{W}$
such that $Z_{L}(X,\mathbf{W})=Y$, then $R_{\mathbf{n}}(X,Y)=+\infty$.
The representation cost $R_{\mathbf{n}}$ naturally describes the
bias induced by the $L_{2}$-regularized loss of DNNs since: 
\[
\min_{\mathbf{W}}C(Z_{L}(X;\mathbf{W}))+\lambda\left\Vert \mathbf{W}\right\Vert ^{2}=\min_{Y}C(Y)+\lambda R_{\mathbf{n}}(Y,X).
\]

\section{Two Reformulations of the Regularized Loss: Hidden Representation
and Covariance Optimization}

We now provide two reformulations of the $L_{2}$-regularized loss
$\mathcal{L}_{\lambda}(\mathbf{W})$ and representation cost $R_{\mathbf{n}}(X,Y)$,
which both put emphasis on the hidden representations $Z_{\ell}$
and how they are progressively modified throughout the neural network.
The first reformulation holds for general non-linearities while the
second only applies to networks with homogeneous nonlinearities.

\subsection{Feature optimization : attraction/repulsion}

\begin{figure}
\center

\subfloat[\textbf{\label{fig:Attraction-Repulsion-ell1-inputs}}Inputs ($\ell=1$).]{\includegraphics[height=3.7cm]{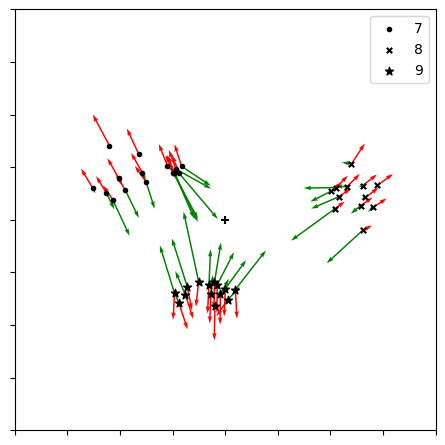}

}\subfloat[\textbf{\label{fig:Attraction-Repulsion-ell1-neurons}}Neurons ($\ell=1$).]{\includegraphics[height=3.7cm]{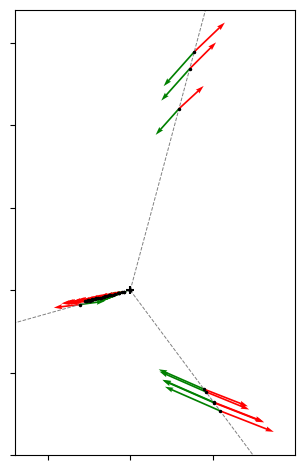}\hspace{0.2cm}

}\subfloat[\textbf{\label{fig:Attraction-Repulsion-ell2-inputs}}Inputs ($\ell=2$).]{\includegraphics[height=3.7cm]{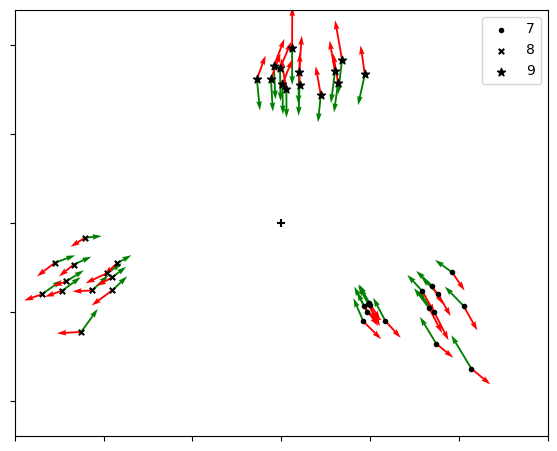}

}\subfloat[\textbf{\label{fig:Attraction-Repulsion-ell2-neurons}}Neurons ($\ell=2$).]{\includegraphics[height=3.7cm]{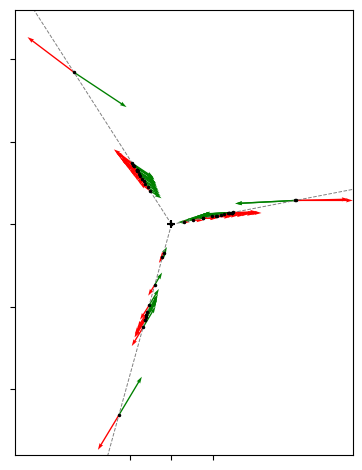}

}

\caption{\textbf{\label{fig:Attraction-Repulsion}Attraction/Repulsion:} Visualization
of the hidden representation $Z_{1}$ and $Z_{2}$ of a $L=3$ ReLU
DNN at the end of training (i.e. after $T=20k$ steps of gradient descent on the original loss $\mathcal{L}_\lambda$) on 3 digits (7,8 and 9) of MNIST \cite{lecun_1998_MNIST} along with the attraction
force in green and repulsion force in red (both forces are approximated
with Tikhonov regularization). For both layers, we plot in \textbf{(a)
}and \textbf{(c)} the PCA of the $N=42$ lines (each corresponding
to a datapoint) and in \textbf{(b) }and \textbf{(d)} the  PCA the
$n_{1}=n_{2}=50$ columns (each corresponding to a neuron in the layer
$\ell=1$ or $\ell=2$). We observe a clustering of the inputs according to their digit, and of the neurons along 3 rays in grey dashed lines.}

\end{figure}

The key observation is that the weights $W_{\ell}$ can be decomposed
as follows:
\[
W_{\ell}=Z_{\ell}\left(Z_{\ell-1}^{\sigma}\right)^{+}+\tilde{W}_{\ell},
\]
where the residual matrix $\tilde{W}_{\ell}$ is orthogonal to $Z_{\ell-1}^{\sigma}$,
i.e. $\tilde{W}_{\ell}Z_{\ell-1}^{\sigma}=0$, and $(\cdot)^{+}$
is the Moore-Penrose pseudo-inverse. This stems from the fact that
$W_{\ell}=W_{\ell}\mathrm{P}_{\mathrm{Im}Z_{\ell-1}^{\sigma}}+\tilde{W}_{\ell}$
where $\tilde{W}_{\ell}:=W_{\ell}\mathrm{P}_{(\mathrm{Im}Z_{\ell-1}^{\sigma})^{\perp}}$,
and $\mathrm{P}_{\mathrm{Im}Z_{\ell-1}^{\sigma}}$, resp. $\mathrm{P}_{(\mathrm{Im}Z_{\ell-1}^{\sigma})^{\perp}}$,
is the orthogonal projection on $\mathrm{Im}Z_{\ell-1}^{\sigma}$,
resp. on the orthogonal complement of $\mathrm{Im}Z_{\ell-1}^{\sigma}$; one
concludes using the facts that $\mathrm{P}_{\mathrm{Im}Z_{\ell-1}^{\sigma}}=Z_{\ell-1}^{\sigma}\left(Z_{\ell-1}^{\sigma}\right)^{+}$
and $Z_{\ell}=W_{\ell}Z_{\ell-1}^{\sigma}$ . 

Note that the matrix $\tilde{W}_{\ell}$ does not affect either the
hidden representations $Z_{\ell}$ nor the output $Z_{L}$. Besides,
the Frobenius norm of $W_{\ell}$ can be rewritten as $\vert\!\vert W_{\ell}\vert\!\vert_{F}^{2}=\vert\!\vert Z_{\ell}\left(Z_{\ell-1}^{\sigma}\right)^{+}\vert\!\vert_{F}^{2}+\vert\!\vert\tilde{W}_{\ell}\vert\!\vert_{F}^{2}$.
When minimizing the $L_{2}$-regularized cost, it is therefore always
optimal to consider null residual matrices $\tilde{W}_{\ell}=0$,
resulting in a reformulation of the cost which only depends on the
pre-activations $Z_{\ell}$: 
\begin{prop}\label{prop:first-reformulation}
The infimum of $\mathcal{L}_{\lambda}(\mathbf{W})=C(Z_{L}(X;\mathbf{W}))+\lambda\left\Vert \mathbf{W}\right\Vert ^{2},$
over the parameters $\mathrm{W}\in\mathbb{R}^{P}$ is equal to the
infimum of
\[
\mathcal{L}_{\lambda}^{r}(Z_{1},\dots,Z_{L})=C(Z_{L})+\lambda\sum_{\ell=1}^{L}\left\Vert Z_{\ell}\left(Z_{\ell-1}^{\sigma}\right)^{+}\right\Vert _{F}^{2}
\]
over the set $\mathcal{Z}$ of hidden representations $\mathbf{Z}=(Z_{\ell})_{\ell=1,\dots,L}$
such that $Z_{\ell}\in\mathbb{R}^{n_{\ell}\times N}$, $\mathrm{Im}Z_{\ell+1}^{T}\subset\mathrm{Im}\left(Z_{\ell}^{\sigma}\right)^{T}$,
with the notations $Z_{0}^{\sigma}=\left(\begin{array}{c}
X\\
\beta\mathbf{1}_{N}^{T}
\end{array}\right)$ and $Z_{\ell}^{\sigma}=\left(\begin{array}{c}
\sigma\left(Z_{\ell}\right)\\
\beta\mathbf{1}_{N}^{T}
\end{array}\right)$.

Furthermore, if $\mathbf{W}$ is a local minimizer of $\mathcal{L}_{\lambda}$
then $(Z_{1}(X; \mathbf{W}),\dots,Z_{L}(X; \mathbf{W}))$ is a local minimizer
of \textup{$\mathcal{L}_{\lambda}^{r}$.} Conversely, keeping the
same notations, if $(Z_{\ell})_{\ell=1,\ldots,L}$ is a local minimizer
of $\mathcal{L}_{\lambda}^{r}$, then $\mathrm{W}=(Z_{\ell}(Z_{\ell-1}^{\sigma})^{+})_{\ell=1,\ldots,L}$
is a local minimizer of $\mathcal{L}_{\lambda}$.
\end{prop}

Note that one can also reformulate the representation cost: 
\[
R_{\mathbf{n}}(X,Y)=\min_{\mathbf{Z}\in\mathcal{Z},Z_{L}=Y}\sum_{\ell=1}^{L}\left\Vert Z_{\ell}\left(Z_{\ell-1}^{\sigma}\right)^{+}\right\Vert _{F}^{2}.
\]
The representation in terms of the output and hidden representations
have several interesting properties, especially when it comes
to minimization: 
\begin{enumerate}
\item The optimization becomes local in the sense that all terms and constraints
depend either only on the output cost $C(Z_{L})$ or on two neighboring
terms (e.g. $\left\Vert Z_{\ell}\left(Z_{\ell-1}^{\sigma}\right)^{+}\right\Vert _{F}^{2}$).
As a result, the (projected) gradient of the loss $\mathcal{L}_{\lambda}^r(Z_{1},...,Z_{L})$
w.r.t. to $Z_{\ell}$ only depends on $Z_{\ell-1},Z_{\ell}$ and $Z_{\ell+1}$.
This is in contrast to the optimization of $\mathcal{L}_{\lambda}(\mathbf{W})$,
where the gradient of $C(Z_{L})$ with respect to $W_{\ell}$ depends
on all parameters $W_{1},\ldots,W_{L}$.

\item The value $\left\Vert Z_{\ell}\left(Z_{\ell-1}^{\sigma}\right)^{+}\right\Vert _{F}^{2}$
represents a 'multiplicative distance' between $Z_{\ell}$ and $Z_{\ell-1}^{\sigma}$
(in contrast to the 'additive distance' $\left\Vert Z_{\ell}-Z_{\ell-1}^{\sigma}\right\Vert _{F}^{2}$);
the representation $Z_{\ell}$ therefore interpolates multiplicatively
between $Z_{\ell-1}$ and $Z_{\ell+1}$. This is most obvious for
linear networks (i.e. $\sigma=\mathrm{id}$ and $\beta=0$): in this
case, one can check that at any global minimizer, the covariances of the hidden layers equal $Z_{\ell}^{T}Z_{\ell}=X^{T}(X^{-T}Z_{L}^{T}Z_{L}X^{-1})^{\frac{\ell}{L}}X$, interpolating between the input covariance $X^{T}X$ and output covariance $Z_{L}^{T}Z_{L}$.

\item A lot of work has been done to propose biologically plausible training methods for DNNs \cite{bengio_2015_biologically_plausible, illing_2019_biologically_plausible}, in contrast to backpropagation which is not local. A line of work \cite{pehlevan_2017_similarity,Obeid_2019_deep_similarity_matching,qin_2021_similarity_matching}, propose a biologically plausible optimization technique which minimizes a cost which closely resembles our first reformulation, with the multiplicative distances $\left\Vert Z_{\ell}\left(Z_{\ell-1}^{\sigma}\right)^{+}\right\Vert _{F}^{2}$  replaced by additive ones $\left\Vert Z_{\ell}-Z_{\ell-1}^{\sigma}\right\Vert _{F}^{2}$. Due to this change, there is no direct correspondence between the networks trained with this biologically plausible technique and those trained with backpropagation. If one could extend this training technique to work with multiplicative distances one could guarantee such a direct correspondence.

\item The optimization leads to an attraction-repulsion algorithm. If we
optimize only on the term $Z_{\ell}$ and fix all other representations,
the only two terms that depend on $Z_{\ell}$ are $\left\Vert Z_{\ell}\left(Z_{\ell-1}^{\sigma}\right)^{+}\right\Vert _{F}^{2}$
and $\left\Vert Z_{\ell+1}\left(Z_{\ell}^{\sigma}\right)^{+}\right\Vert _{F}^{2}$.
The former term is attractive as it pushes the representations $Z_{\ell}$
towards the origin (and hence pushes the representations at depth
$\ell$ of every input towards each other), especially along directions
where $Z_{\ell-1}^{\sigma}$ is small. The latter term is repulsive
as it pushes the representations $Z_{\ell}^{\sigma}$ away from the
origin, especially along directions where $Z_{\ell+1}$ is large.
\item This attraction-repulsion process is similar to the Information
Bottleneck theory \cite{tishby2015_info_bottleneck}: the repulsive
term ensure that $Z_{\ell}$ keeps enough information about the inputs
to reconstruct $Z_{\ell+1}$, while the attractive term pushes
$Z_{\ell}$ to keep as little information as possible.
\end{enumerate}
The attraction and repulsion forces of of the $\ell$-th layer are
the derivative $\partial_{Z_{\ell}}\left\Vert Z_{\ell}\left(Z_{\ell-1}^{\sigma}\right)^{+}\right\Vert _{F}^{2}$
and $\partial_{Z_{\ell}}\left\Vert Z_{\ell+1}\left(Z_{\ell}^{\sigma}\right)^{+}\right\Vert _{F}^{2}$
which are both $n_{\ell}\times N$ matrices. One can visualize these
forces either column by column (each column corresponding to a datapoint
$i=1,\dots,N$) or line by line (each line corresponding to a neuron
$k=1,\dots,n_{\ell}$). These two visualizations of the forces are
presented in Figure \ref{fig:Attraction-Repulsion} for the two hidden
layers of a depth $L=3$ network, projected to the 2 largest principal
components of the columns resp. lines of $Z_{\ell}$. Figures \ref{fig:Attraction-Repulsion-ell1-inputs}
and \ref{fig:Attraction-Repulsion-ell2-inputs} illustrate how the
inputs corresponding to different classes are pushed away from each
other, leading to a clustering effect. Figures \ref{fig:Attraction-Repulsion-ell1-neurons}
and \ref{fig:Attraction-Repulsion-ell2-neurons} show that the neurons
naturally align along rays starting from the origin. This happens
for homogeneous non-linearities, such as the ReLU in this example,
because if two neurons $k,k'$ have proportional activations, i.e.
$Z_{\ell,k}=\alpha Z_{\ell,k}$ for some $\alpha\in\mathbb{R}$, then
their attractive and repulsive forces will also be proportional with
the same scaling $\alpha$. As a result, the neuron $k$ is stable,
i.e. the attraction and repulsion cancel each other, if and only if
the neuron $k'$ is stable.

This phenomenon can be interpreted as a form of sparsity: a group
of aligned neurons can be replaced by a single neuron without changing
the resulting function $Z_{L}$. Can we guarantee a degree of sparsity
in the hidden representations? Can we bound the number of aligned
groups in a neuron? In the next section, we introduce a further reformulation
of the loss which allows us to partially answer these questions.

\subsection{Covariance learning : partial convex optimization for positively
homogeneous nonlinearities}

The loss of the first reformulation $\mathcal{L}_{\lambda}^{r}$ depends
on the hidden representations $Z_{\ell}$ and $Z_{\ell}^{\sigma}$
only through the covariances $K_{\ell}=Z_{\ell}^{T}Z_{\ell}$ and
$K_{\ell}^{\sigma}=\left(Z_{\ell}^{\sigma}\right)^{T}Z_{\ell}^{\sigma}$,
since $\left\Vert Z_{\ell}\left(Z_{\ell-1}^{\sigma}\right)^{+}\right\Vert _{F}^{2}=\mathrm{Tr}\left[K_{\ell}\left(K_{\ell-1}^{\sigma}\right)^{+}\right]$.
Hence, we provide a second reformulation expressed in terms of the
tuple of covariance pairs $\mathbf{K}=((K_{1},K_{1}^{\sigma}),\dots,(K_{L-1},K_{L-1}^{\sigma}))$
and the outputs $Z_{L}$. Using the notations $K_{0}^{\sigma}=X^{T}X+\beta^{2}\mathbf{1}_{N\times N}$
and $K_{L}=Z_{L}^{T}Z_{L}$, we define: 
\[
\mathcal{L}_{\lambda}^{\mathrm{k}}(\mathbf{K},Z_{L})=C(Z_{L})+\lambda\sum_{\ell=1}^{L}\mathrm{Tr}\left[K_{\ell}\left(K_{\ell-1}^{\sigma}\right)^{+}\right].
\]
It remains to identify the set $\mathcal{K}_{\mathbf{n}}(X)$ of covariances
$\mathbf{K}$ and outputs $Z_{L}$ which can be represented by a width
$\mathbf{n}$ network with inputs $X$. For positively homogeneous
nonlinearities of degree 1 such as the ReLU (i.e. when $\sigma(\lambda x)=\lambda\sigma(x)$
for any positive $\lambda$), the set $\mathcal{K}_{\mathbf{n}}(X)$
can be expressed using the notion of conical hulls. 
\begin{defn}
The conical hull of $\Omega\subset\mathbb{R}^{d}$ is the set $\mathrm{cone}\left(\Omega\right):=\left\{ \sum_{i=1}^{k}\alpha_{i}\omega_{i}:k\geq0,\alpha_{i}\geq0,\omega_{i}\in\Omega\right\} $
and its $m$-conical hull for $m\geq1$ is the set $\mathrm{cone}_{m}\left(\Omega\right):=\left\{ \sum_{i=1}^{m}\alpha_{i}\omega_{i}:\alpha_{i}\geq0,\omega_{i}\in\Omega\right\} $.

Note that by Caratheodory's theorem for conical hulls, $\mathrm{cone}_{m}\left(\Omega\right)=\mathrm{cone}\left(\Omega\right)$
for any $m\geq d$. We now proceed to the description of the set $\mathcal{K}_{\mathbf{n}}(X)$
and obtain the second formulation of the $L_{2}$ regularized loss
and of the representation loss $R_{n}$ in terms of covariances: 
\end{defn}

\begin{prop}
\label{prop:secondrepresentation} For positively homogeneous non-linearities
$\sigma$, the infimum of $\mathcal{L}_{\lambda}(\mathbf{W})=C(Z_{L}(X;\mathbf{W}))+\lambda\left\Vert \mathbf{W}\right\Vert ^{2},$
over the parameters $\mathrm{W}\in\mathbb{R}^{P}$ is equal to the
infimum over $\mathcal{K}_{\mathbf{n}}(X)$ of
\[
\mathcal{L}_{\lambda}^{\mathrm{k}}(\mathbf{K},Z_{L})=C(Z_{L})+\lambda\sum_{\ell=1}^{L}\mathrm{Tr}\left[K_{\ell}\left(K_{\ell-1}^{\sigma}\right)^{+}\right].
\]
The set $\mathcal{K}_{\mathbf{n}}(X)$ is the set of covariances $\mathbf{K}=((K_{1},K_{1}^{\sigma}),\dots,(K_{L-1},K_{L-1}^{\sigma}))$
and outputs $Z_{L}$ such that for all hidden layer $\ell=1,\dots,L-1$: 
\begin{itemize}
\item the pair $\left(K_{\ell},K_{\ell}^{\sigma}\right)$ belongs to the
(translated) $n_{\ell}$-conical hull 
\[
S_{n_{\ell}}=\mathrm{cone}_{n_{\ell}}\left( \left\{ \left(xx^{T},\sigma(x)\sigma(x)^{T}\right):x\in\mathbb{R}^{N}\right\} \right) +(0,\beta^{2}\mathbf{1}_{N\times N}),
\]
\item $\mathrm{Im}K_{\ell}\subset\mathrm{Im}K_{\ell-1}^{\sigma}$, with
the notation $K_{0}^{\sigma}=X^{T}X+\beta^{2}\mathbf{1}_{N\times N}$,
and  $\mathrm{Im}Z_{L}\subset\mathrm{Im}K_{L-1}^{\sigma}.$ 
\end{itemize}
\end{prop}

Note that one can also reformulate the representation cost: 
\[
R_{\mathbf{n}}(X,Y)=\min_{\mathbf{K}:(\mathbf{K},Y)\in\mathcal{K}_{\mathbf{n}}(X)}\sum_{\ell=1}^{L}\mathrm{Tr}\left[K_{\ell}\left(K_{\ell-1}^{\sigma}\right)^{+}\right].
\]

\begin{rem}
\label{rem:criticalpoints}In contrast to the first loss $\mathcal{L}_{\lambda}^{r}$
whose local minima were in correspondence with the local minima of
the original loss $\mathcal{L}_{\lambda}$, the second loss $\mathcal{L}_{\lambda}^{\mathrm{k}}$
can in some cases have strictly less critical points and local minima.
Indeed, since the map $\mathbf{W}\mapsto\left(\mathbf{K},Z_{L}\right)$
is continuous, if $(\mathbf{K}(\mathrm{W}),Z_{L}(\mathbf{W}))$ is
a local minimum, then so is $\mathbf{W}$. However, the converse is
not true: we provide a counterexample in the Appendix, i.e. a set
of weights $\mathbf{W}$ of a depth $L=2$ network which is a local
minimum of $\mathcal{L}_{\lambda}$ and such that the corresponding
$(\mathbf{K},Z_{L})$ is not a local minimum of $\mathcal{L}_{\lambda}^{\mathrm{k}}$.
\end{rem}

Since the dimension of the space of pairs of symmetric $N\times N$
matrices is $N(N+1)$, if $n_{\ell}\geq N(N+1)$ then, by the Caratheodory's
theorem for conical hulls, 
\[
S_{n_{\ell}}=S:=\mathrm{cone}\left( \left\{ \left(xx^{T},\sigma(x)\sigma(x)^{T}\right):x\in\mathbb{S}^{N-1}\right\} \right) +(0,\beta^{2}\mathbf{1}_{N\times N}).
\]
Hence, as soon as $n_{\ell}\geq N(N+1)$ for all hidden layers, the
set $\mathcal{K}_{\mathbf{n}}(X)$ does not depend on the list of
widths $\mathbf{n}$. We denote 
\[
\mathcal{K}(X)=\left\{ (\mathbf{K},Z_{L})\mid\forall\ell=1,\ldots,L-1,(K_{\ell},K_{\ell}^{\sigma})\in S,\mathrm{Im}K_{\ell}\subset\mathrm{Im}K_{\ell-1}^{\sigma},\mathrm{Im}Z_{L}\subset\mathrm{Im}K_{L-1}^{\sigma}\right\} 
\]
this width-independent set. The following proposition shows that for
sufficiently wide networks with a positively homogeneous nonlinearity,
training a deep DNN with $L_{2}$ regularization is  equivalent to
a partially convex optimization over a translated convex cone.
\begin{prop}
The set $\mathcal{K}(X)$ is a translated convex cone: after the
suitable translation, it is equal to its conical hull. The cost $\mathcal{L}_{\lambda}^{\mathrm{k}}\left(\mathbf{K},Z_{L}\right)$
is partially convex w.r.t. to the outputs $Z_{L}$ and the pairs $(K_{\ell},K_{\ell}^{\sigma})$,
i.e. it is convex if one fixes the other parameters and let only $(K_{\ell},K_{\ell}^{\sigma})$,
or $Z_{L}$, vary.
\end{prop}

\subsection{Direct optimization of the reformulations}

It is natural at this point to wonder whether one could optimize directly
over the representations $\mathbf{Z}$ (using the first reformulation)
or over the covariances $\mathbf{K}$ and output $Z_{L}$ (using the
second reformulation), and whether this would have an advantage over
the traditional optimization of the weights.

For the first reformulation, one can simply use the projected gradient
descent with updates given by
\[
\mathbf{Z}_{t+1}=P_{\mathcal{Z}}\left(\mathbf{Z}_{t}-\eta\nabla\mathcal{L}_{\lambda}^{r}(\mathbf{Z})\right)
\]
for any projection $P_{\mathcal{Z}}$ to the constraint space $\mathcal{Z}$.
For example, a projection is obtained by mapping $Z_{\ell}$ to $Z_{\ell}P_{\mathrm{Im}Z_{\ell-1}^{\sigma}}$
sequentially from $\ell=1$ to $\ell=L$. Note however that the loss explodes as the constraints become unsatisfied so that for gradient flow, there is no need for the projections. This suggests that these projections might also be unnecessary as long as the learning rate is small enough. For more details, see Appendix \ref{subsec:optimization-projection}.

For the second reformulation, there is no obvious way to compute a
projection to the constraint space $\mathcal{K}$: the cone $S$ is
spanned by an infinite amount of points and we do not have an explicit
formula for the dual cone $S^{*}$. Frank-Wolfe optimization can be
used to overcome the need for computing the projections.

However, these direct optimizations of the reformulations lead to
issues of computational complexity and stability. First, the computation
of the gradients $\nabla\mathcal{L}_{\lambda}^{r}(\mathbf{Z})$ and
$\nabla\mathcal{L}_{\lambda}^{\mathrm{k}}(\mathbf{K},Z_{L})$ requires
solving a linear equation of dimension $N$, which is very costly,
in contrast to the traditional optimization of the weights $\mathbf{W}$
for which the gradient can be computed very efficiently. Second, if
$Z_{\ell}^{\sigma}$ is not full-rank, the computation of its pseudo-inverse
$\left(Z_{\ell}^{\sigma}\right)^{+}$ and the projection $P_{\mathrm{Im}Z_{\ell}^{\sigma}}$
are very unstable. Therefore, if we only have finite-precision knowledge
of $Z_{\ell}$, we cannot reliably compute $\left(Z_{\ell}^{\sigma}\right)^{+}$
nor $P_{\mathrm{Im}Z_{\ell}^{\sigma}}$.

Although it could be possible to solve these problems (e.g. using
the Tikhonov regularization for the unstability problem) and to develop
efficient algorithms to optimize both reformulations efficiently,
we decided in this paper to focus on the theoretical implications
of these reformulations. 

\section{Sparsity of the Regularized Optimum for Homogeneous DNNs}

In this section, we assume that the non-linearity is positively homogenous.
Under this assumption, the second reformulation of the loss (and of
the representation cost) holds and implies the existence of a sparsity
phenomenon.

First observe that as the widths $\mathbf{n}$ increase, both the
global minimizer of the loss $\min_{\mathbf{W}}\mathcal{L}_{\lambda}(\mathbf{W})$
and the representation cost $R_{\mathbf{n}}(X,Y)$ diminish. We denote
by $\mathcal{L}_{\lambda,\mathbf{n}}$ the $L_{2}$-regularized loss
of DNNs with widths $\mathrm{n}$. Recall that the depth $L$ is fixed. 
\begin{prop}
If $\mathbf{n}\leq\mathbf{n}'$ (in the sense that $n_{\ell}\leq n'_{\ell}$
for all $\ell$ and $n_{0}=n'_{0}$ and $n_{L}=n'_{L}$), then
\[
\min_{\mathbf{W}\in\mathbb{R}^{P_{\mathbf{n}}}}\mathcal{L}_{\lambda,\mathbf{n}}(\mathbf{W})\geq\min_{\mathbf{W}\in\mathbb{R}^{P_{\mathbf{n}'}}}\mathcal{L}_{\lambda,\mathbf{n}'}(\mathbf{W})
\]
and for any $X\in\mathbb{R}^{n_{0}\times N}$ and $Y\in\mathbb{R}^{n_{L}\times N}$,
$R_{\mathbf{n}}(X,Y)\geq R_{\mathbf{n}'}(X,Y).$
\end{prop}

\begin{proof}
Let us assume that the parameters $\mathbf{W}^{*}$ are optimal for
a width $\mathbf{n}$ network, the parameters can be mapped to parameters
of a wider network by adding `dead' neurons (i.e. neurons with zero
incoming and outcoming weights) without changing the network
function $f_{\mathbf{W}^{*}}$ nor the norm of the parameters $\left\Vert W\right\Vert $. 
\end{proof}
For DNNs with positively homogeneous nonlinearities, a direct consequence
of our reformulation through the hidden covariances is that both the
global minimum of the loss $\mathcal{L}_{\lambda}$ and the representation
cost $R_{\mathbf{n}}(X,Y)$ plateau for any widths $\mathbf{n}$ such
that $n_{\ell}\geq N(N+1)$. 
\begin{prop}
\label{prop:upper_bound_plateau} For any positively homogeneous nonlinearity
$\sigma$, any widths $\mathbf{n}$ and $\mathbf{n}'$ such that $n_{0}=n'_{0}$,
$n_{L}=n'_{L}$ and for all $\ell=1,\dots,L-1$, $n_{\ell},n'_{\ell}\geq N(N+1)$,
for all $\lambda>0$, we have: 
\begin{align*}
\min_{\mathbf{W}\in\mathbb{R}^{P_{\mathbf{n}}}}\mathcal{L}_{\lambda,\mathbf{n}}(\mathbf{W}) & =\min_{\mathbf{W}\in\mathbb{R}^{P_{\mathbf{n}'}}}\mathcal{L}_{\lambda,\mathbf{n}'}(\mathbf{W}).
\end{align*}
Under the same conditions $R_{\mathbf{n}}(X,Y)=R_{\mathbf{n}'}(X,Y)$ for any $X\in\mathbb{R}^{n_{0}\times N}$
and $Y\in\mathbb{R}^{n_{L}\times N}$,.

\end{prop}

\begin{proof}
This follows directly from the second reformulation and the fact that
by Caratheodory's theorem for conical hulls, $S_{n}=S$ if $n\geq N(N+1)$
(see discussion after Remark \ref{rem:criticalpoints}). 
\end{proof}
We can therefore define a width-independent representation cost $R(X,Y)$
(which still depends on the fixed depth $L$) equal to the representation
cost $R_{\mathbf{n}}(X,Y)$ of any sufficiently wide network.

\begin{figure}
\center

\includegraphics[height=4cm]{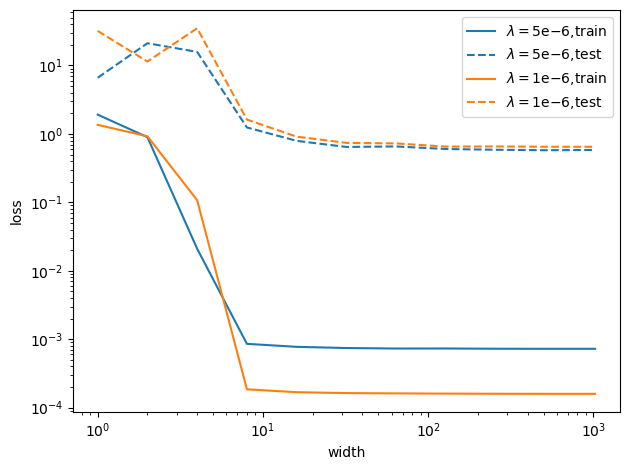}
\hspace{0.3cm}
\includegraphics[height=4cm]{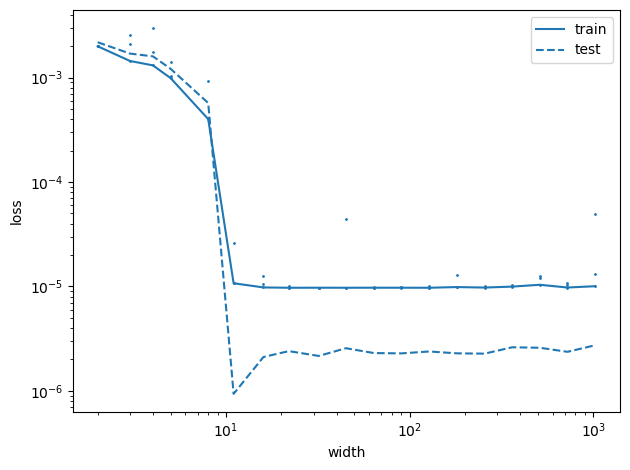}

\caption{\label{fig:Loss-plateau}\textbf{Loss plateau:} Plots of the train
loss (full lines) and test loss (dashed lines) as a function of the
width for depth $L=3$ DNNs for different datasets: (left) cross-entropy loss for a subset of MNIST ($N=1000$)
and two values of $\lambda$; (right) MSE with $\lambda=10^{-6}$ on $N=1000$ Gaussian inputs and outputs evaluated on a fixed teacher network of depth $L=3$  and width 10( (right). For the right
plot we took (the minimum is taken over 3 independent trials, represented by
the small blue dots). In both settings, the plateau appears to start around
10, much earlier than $N^2=10^6$. The regularization term is included in the training loss but not the test, leading to a smaller test loss on the right.}
\end{figure}

\subsection{Rank of the Hidden Representations}

Now that we have revealed the plateau phenomenon, a natural question
that we investigate in this section is when does this plateau begin.
In order to do so, we introduce the notion of rank $\mathrm{Rank}_{\sigma}(K,K^{\sigma})$
of a pair of Gram matrices $(K,K^{\sigma})\in S$ which is the minimal
number $k$ such that 
\begin{equation}
K=\sum_{i=1}^{k}z_{i}z_{i}^{T}\text{ and }K^{\sigma}=\sum_{i=1}^{k}\sigma\left(z_{i}\right)\sigma\left(z_{i}\right)^{T}+\beta^{2}\mathbf{1}_{N\times N}\label{eq:Rank-Signa}
\end{equation}
for some $z_{1},\dots,z_{k}\in\mathbb{R}^{N}$. This notion of rank
describes exactly the minimal number of neurons required to recover
a set of covariances $\mathbf{K}$:
\begin{prop}\label{prop:rank_is_start_of_plateau}
Let $(\mathbf{K},Z_L)\in\mathcal{K}(X)$,
then there are parameters $\mathbf{W}$ of a width $\mathbf{n}$ network
with covariances and outputs $\mathbf{K}$ if and only if $n_{\ell}\geq\mathrm{Rank}_{\sigma}\left(K_{\ell},K_{\ell}^{\sigma}\right)$
for all $\ell=1,\dots,L-1$.
\end{prop}

We can now describe the plateau $R=\left\{ \mathbf{n}:\min_{\mathbf{W}\in\mathbb{R}^{P_{\mathbf{n}}}}\mathcal{L}_{\lambda,\mathbf{n}}(\mathbf{W})=\min_{\mathbf{m}}\min_{\mathbf{W}\in\mathbb{R}^{P_{\mathbf{m}}}}\mathcal{L}_{\lambda,\mathbf{m}}(\mathbf{W})\right\} $,
i.e. the set of widths $\mathbf{n}$ such that the minimum $\min_{\mathbf{W}\in\mathbb{R}^{P_{\mathbf{n}}}}\mathcal{L}_{\lambda,\mathbf{n}}(\mathbf{W})$
is optimal over all possible widths:
\begin{cor}
Let $K_{\min}$ be the set of covariances sequences $(\mathbf{K},Z_L)$
which are global minima of the second reformulation. We have that
$\mathbf{n}\in R$ if and only if there is a $(\mathbf{K},Z_L)\in K_{\min}$
such that $n_{\ell}\geq\mathrm{Rank}_{\sigma}\left(K_{\ell},K_{\ell}^{\sigma}\right)$. 
\end{cor}

Hence, the investigation of $\mathrm{Rank}_{\sigma}(\cdot,\cdot)$
is crucial to understand where the plateau begins; unfortunately,
it can be difficult to compute. However, from its definition and the
Caratheodory's theorem for conical hulls (see our discussion after
Remark \ref{rem:criticalpoints}), we have the following natural bounds: 
\begin{lem}
\label{lem:bound_rank}For any pair $(K,K^{\sigma})\in S$, we have
$\mathrm{Rank}\left(K\right)\leq\mathrm{Rank}_{\sigma}(K,K^{\sigma})\leq N(N+1)$. 
\end{lem}

We show in the next section that the order of magnitude of the upper
bound is tight. More specifically, we construct a dataset for which
any global optimum satisfies
$\mathrm{Rank}_{\sigma}(K_{1},K_{1}^{\sigma})\geq\nicefrac{N^{2}}{4}$.
This implies that, in this example, the plateau transition occurs
when the number of hidden neurons is of order $O(N^{2})$. Note however
that, in our numerical experiments (see Figure \ref{fig:Loss-plateau}),
the rank of the global optimum can be much smaller for more traditional
dataset such as MNIST.
\begin{rem}
The start of the plateau measures a notion of sparsity of the learned
network, since the networks learned in the plateau are equivalent
to a network at the start of the plateau, i.e. large networks are
equivalent in terms of their covariances and outputs $(\mathbf{K},Z_L)$
to a (potentially much) smaller network.

Even though the set of pairs $(K_\ell,K^\sigma_\ell)$ in the cone $S$ that are not full rank has measure zero, the optimal representations $(K_\ell,K^\sigma_\ell)$ always lie on the border of the cone $S$  (since the derivative of the cost $\mathcal{L}_\lambda^k$ w.r.t. $(K_\ell,K^\sigma_\ell)$ never vanishes) where the rank is lower. More precisely, the rank is determined by the dimension of the smallest face that contains the optimum (e.g. the pairs $(K_\ell,K^\sigma_\ell)$ on the edges of $S$ have rank at most 2 for example, while those on the vertices are rank 1).

We can identify different degrees of sparsity depending on how the rank of the hidden representations
scales with the number of datapoints $N$: if the rank is $o(N)$
the covariances $K_{\ell},K_{\ell}^{\sigma}$ are low-rank (in the traditional linear sense) and for
shallow networks the effective number of parameters (i.e. the number
of parameters at the start of the plateau) is $o(N^2)$, if the rank
is $o(\sqrt{N})$ then for deep networks the effective number of parameters
$o(N)$. This could explain why very large networks with `too many
parameters' are able to generalize, since their effective number of
parameters is of the order of the number datapoints. Very large networks
can therefore be trained safely knowing that thanks to $L_{2}$-regularization,
the network is able to recognize what is the 'right' width of the
network.
\end{rem}

\subsection{Tightness of the Upper-Bound}

In this section, we construct a pair of input and output datasets
$X$ and $Y$, both in $\mathbb{R}^{N\times N}$, such that for any
optimal parameters $\mathbf{W}$ of a ReLU network of depth $L=2$
with no bias ($\beta=0$), the rank of the hidden representation $\mathrm{Rank}_{\sigma}\left(K_{1},K_{1}^{\sigma}\right)$
is greater than $\nicefrac{N^{2}}{4}$. 

Note that one can write the decomposition (\ref{eq:Rank-Signa}) as
$K=C^{T}C$ and $K_{\sigma}=B^{T}B$ where $C=(z_{1},\ldots,z_{k})$
and $B=\mathrm{ReLU}(C)$ is obtained by applying elementwise the
ReLU to $C$. Key to our construction is the fact that $B$ is then
a matrix with non-negative entries: the matrix $K_{\sigma}$ is completely
positive and $\mathrm{Rank}_{\sigma}\left(K,K^{\sigma}\right)$ can
be studied using the CP-rank of $K$: 
\begin{defn}
A $N\times N$ matrix $A$ is completely positive if $A=B^{T}B$ for
a $k\times N$ matrix $B$ with non-negative entries. The CP-rank
$\mathrm{Rank_{cp}}\left(A\right)$ of a completely positive matrix
$A$ is the minimal integer $k$ such $A=B^{T}B$ for a $k\times N$
matrix $B$ with non-negative entries.
\end{defn}

When $\sigma$ is the ReLU, the kernel $K_{\ell}^{\sigma}$ is completely
positive for all hidden layers $\ell$, and thus 
\[
\max\left(\mathrm{Rank}\left(K_{\ell}\right),\mathrm{Rank}_{cp}\left(K_{\ell}^{\sigma}\right)\right)\leq\mathrm{Rank}_{\sigma}(K,K^{\sigma}).
\]
In order to obtain the tightness of the upper bound, we proceed in
two steps: first, we construct a completely positive matrix $A$ with
high CP-rank, and then construct inputs $X$ and outputs $Y$ such
that the optimal hidden covariance $K_{1}=K_{1}^{\sigma}$ for a depth
$L=2$ network equals the matrix $A$.

As shown in \cite{drew1994completely}, bi-partite graphs can be used
to construct matrices with high CP-rank. We refine this by
showing that graphs on $N$ vertices without cliques of $3$ or more vertices
lead to $N \times N$ matrices with CP-rank equal to the number of edges, and
as a corollary, we construct a completely positive matrix with CP-rank
equal to $\nicefrac{N^{2}}{4}$. 
\begin{prop}\label{prop:CP_rank_graph} 
Given a graph $G$ with $N$ vertices and $k$ edges, consider the
$k\times N$ matrix $E$ with entries $E_{ev}=1$ if the vertex $v$
is an endpoint of the edge $e$ and $E_{ev}=0$ otherwise. The matrix
$A=E^{T}E$ is completely positive and if the graph $G$ contains
no cliques of 3 or more vertices then $\mathrm{Rank_{cp}}\left(A\right)=k$.
\end{prop}

Hence, to obtain a completely positive matrix of high CP-rank, it
remains to find a graph with no cliques and as many edges as possible.
For even $N$, we consider the complete bipartite graph, i.e. the
graph with two groups of size $N/2$ and with edges between any two
vertices iff they belong to different groups. For this graph, the
matrix $B_N=E^{T}E$ takes the form of a block matrix: 
\[
B_{N}=\left(\begin{array}{cc}
\frac{N}{2}I_{\frac{N}{2}} & \mathbf{1}_{\frac{N}{2}\times\frac{N}{2}}\\
\mathbf{1}_{\frac{N}{2}\times\frac{N}{2}} & \frac{N}{2}I_{\frac{N}{2}}
\end{array}\right)
\]
where $\mathbf{1}_{\frac{N}{2}}$ is the $N/2\times N/2$ matrix with
all ones entries. Since this bipartite graph has no cliques and $\nicefrac{N^{2}}{4}$
edges, from the previous proposition, we obtain $\mathrm{Rank_{cp}}\left(B_{N}\right)=\frac{N^{2}}{4}$.

The following proposition shows how for any completely positive matrix (with CP-rank $k$) there is a dataset such that a shallow ReLU network will have a hidden representation pair $(K_1,K_1^\sigma)$of rank $k$:
\begin{prop}
\label{prop:example-bipartite-tight} 
Consider a width-$\mathrm{n}$ shallow network ($L=2$) with ReLU activation, no bias $\beta=0$, $n_{0}=N$,
$n_{1}\geq N(N+1)$, input dataset $X_{N}=I_{N}$, and any output dataset
$Y_{N}$ such that $\left(Y_N^T Y_N \right)^{\frac{1}{2}}$ is a completely positive matrix with CP-rank $k$.

At any global minimum of $R_{\mathrm{n}}(X_{N},Y_{N})$, we have $\mathrm{Rank}_{\sigma}\left(K_{1},K_{1}^{\sigma}\right)=k$.
Furthermore for $\lambda$ small enough, at any global minimum of
$\mathcal{L}_{\lambda,\mathrm{n}}^{\mathrm{MSE}}(\mathbf{W})=\frac{1}{N}\left\Vert Y(X_{N};\mathbf{W})-Y_{N}\right\Vert _{F}^{2}+\lambda\left\Vert \mathbf{W}\right\Vert ^{2},$
we have $\mathrm{Rank}_{\sigma}\left(K_{1},K_{1}^{\sigma}\right)\geq k$.
%
\end{prop}

By Proposition \ref{prop:example-bipartite-tight}, with the outputs $Y_N=B_N$, the rank of the hidden representations (and the start the plateau) is larger or equal to $\frac{N^2}{4}$. This shows that the order  $N^2$ of the bound of 
Lemma \ref{lem:bound_rank} is tight when it comes to data-agnostic bounds. However under certain assumptions on the data one can guarantee a much earlier plateau.

For example, if we instead apply Proposition \ref{prop:example-bipartite-tight} to a task closer to classification, where the columns of the outputs $Y_N \in \mathbb{R}^{n_L \times N}$ are one-hot vectors, then $(Y_N^T Y_N)^{\frac 1 2}$ is (up to permutations of the columns/lines) a block diagonal matrix with $n_L$ constant positive blocks, which is completely positive with CP-rank equal to the number of classes $n_L$. This is in line with our empirical experiments in Figure \ref{fig:Loss-plateau} where we observe in MNIST a plateau starting roughly at a width of 10, which is the number of classes.

Another example where the structure of the data leads to an earlier plateau is when the input and output dimensions are both 1, in which case we can guarantee that the start of the plateau grows at most linearly with the number of datapoints $N$:
\begin{prop}\label{prop:shallow_1D_plateau}
Consider shallow networks ($L=2$) with scalar inputs and outputs
($n_{0}=n_{2}=1$), a ReLU nonlinearity, and a dataset $X,Y\in\mathbb{R}^{1\times N}$.
Both the representation cost $R_{\mathbf{n}}(X,Y)$ and global minimum
$\min_{\mathbf{W}}\mathcal{L}_{\lambda,\mathbf{n}}(\mathbf{W})$ for
any $\lambda>0$ are independent of the width $n_1$ as long as $n_{1}\geq4N$.
\end{prop}

More generally, we propose to view the start of the plateau as an indicator of how well a certain task is adapted to a DNN architecture. An early plateau suggests that the network is able to solve the task optimally with very few neurons, in contrast to a late plateau. The fact that the optimal network requires few neurons (and hence few parameters) can be used to guarantee good generalization.

\subsection{Conclusion}

We have given two reformulations of the loss of $L_{2}$-regularized
DNNs. The first works for a general non-linearity and shows how the
hidden representations of the inputs $Z_{1},\dots,Z_{L-1}$ are learned
to interpolate between the input and output representations, as a balance
between attraction and repulsion forces for every layer. The second
reformulation for homogeneous non-linearities allows us to analyze
a sparsity effect of $L_{2}$-regularized DNNs, where the learned
networks are equivalent to another network with much fewer neurons.
This effect can be visualized by the appearance of a plateau in the
minimal loss as the number of neurons grows, the earlier the plateau,
the sparser the solution, since an early plateau means that very few
neurons were required to obtain the same loss as a network with an
infinite number of neurons. We show that this plateau cannot start
later than $N(N+1)$, and then show that the order of this bound is
tight by constructing a toy dataset for which the plateau starts at
$\nicefrac{N^{2}}{4}$, however, we observe that on more traditional
datasets, the start of the plateau can be much earlier.

\section*{Acknowledgements}
C. Hongler acknowledges support from the Blavatnik Family Foundation, the Latsis Foundation, and the NCCR Swissmap.

\bibliographystyle{plain}
\bibliography{./../main}

\section*{Checklist}

\begin{enumerate}

\item For all authors...
\begin{enumerate}
  \item Do the main claims made in the abstract and introduction accurately reflect the paper's contributions and scope?
    \answerYes{}
  \item Did you describe the limitations of your work?
    \answerYes{In section 3.3 we discuss the limitations of our reformulations when it comes to numeric optimization.}
  \item Did you discuss any potential negative societal impacts of your work?
    \answerNo{This work is theoretical and has no direct societal impact.}
  \item Have you read the ethics review guidelines and ensured that your paper conforms to them?
    \answerYes{}
\end{enumerate}

\item If you are including theoretical results...
\begin{enumerate}
  \item Did you state the full set of assumptions of all theoretical results?
    \answerYes{}
        \item Did you include complete proofs of all theoretical results?
    \answerYes{In the appendix.}
\end{enumerate}

\item If you ran experiments...
\begin{enumerate}
  \item Did you include the code, data, and instructions needed to reproduce the main experimental results (either in the supplemental material or as a URL)?
    \answerYes{In the supplementary material.}
  \item Did you specify all the training details (e.g., data splits, hyperparameters, how they were chosen)?
    \answerYes{In the appendix.}
        \item Did you report error bars (e.g., with respect to the random seed after running experiments multiple times)?
    \answerYes{(partly) In Figure 2 right we show the 3 different trials over which we have taken the minimum. However in general the plots are only given to illustrate the theoretical results not to support our argument, we therefore valued simplicity and readability.}
        \item Did you include the total amount of compute and the type of resources used (e.g., type of GPUs, internal cluster, or cloud provider)?
    \answerYes{In the appendix.}
\end{enumerate}

\item If you are using existing assets (e.g., code, data, models) or curating/releasing new assets...
\begin{enumerate}
  \item If your work uses existing assets, did you cite the creators?
    \answerYes{}
  \item Did you mention the license of the assets?
    \answerYes{In the appendix.}
  \item Did you include any new assets either in the supplemental material or as a URL?
    \answerNA{}
  \item Did you discuss whether and how consent was obtained from people whose data you're using/curating?
    \answerNA{}
  \item Did you discuss whether the data you are using/curating contains personally identifiable information or offensive content?
    \answerNA{}
\end{enumerate}

\item If you used crowdsourcing or conducted research with human subjects...
\begin{enumerate}
  \item Did you include the full text of instructions given to participants and screenshots, if applicable?
    \answerNA{}
  \item Did you describe any potential participant risks, with links to Institutional Review Board (IRB) approvals, if applicable?
    \answerNA{}
  \item Did you include the estimated hourly wage paid to participants and the total amount spent on participant compensation?
    \answerNA{}
\end{enumerate}

\end{enumerate}

\newpage

\appendix

The appendix is structured as follows:
\begin{enumerate}
\item In Section \ref{sec:app-Experimental-Setup}, we describe the Experimental
setup.
\item In Section \ref{subsec:app-first-reformulation}, we prove Proposition
\ref{prop:first-reformulation} of the main underlying the first reformulation.
\item In Section \ref{subsec:app-second-reformulation}, we prove Proposition
\ref{prop:secondrepresentation} for the second reformulation. We also give an example of a local
minimum of the original loss which is not a local minimum in the second
reformulation.
\item In Section \ref{sec:app-plateau}, we prove Proposition \ref{prop:rank_is_start_of_plateau}, \ref{prop:CP_rank_graph}, \ref{prop:example-bipartite-tight} and \ref{prop:shallow_1D_plateau} of the main. 
\end{enumerate}

\section{Experimental Setup \label{sec:app-Experimental-Setup}}

The experiments were done on fully-connected DNNs of depth $L=3$
with varying widths.

We used the MNIST dataset \cite{lecun_1998_MNIST} under the 'Creative
Commons Attribution-Share Alike 3.0' license. For the MNIST examples
we trained the networks on the multiclass cross-entropy loss with
$L_{2}$-regularization.

We also used synthetic data sampled from a teacher network. The network
has depth $L=3$, widths $\mathbf{n}=(50,10,10,10)$ with random Gaussian
weights. The cost used was the Mean Squared Error (MSE).

For the experiments of Figure \ref{fig:Attraction-Repulsion} of the main, the DNN was trained with
full batch GD. For the experiments of Figure \ref{fig:Loss-plateau} we first trained with
Adam \cite{adam} and finished with full batch GD (GD seems to be
better suited to consistently reach the bottom of the local minima,
though Adam trains faster overall). For the right plot of Figure \ref{fig:Loss-plateau},
three independent networks were trained for every width and the one
with the smallest loss at the end of training was chosen (the plotted
test error is that of the chosen network).

The goal of Figure \ref{fig:Loss-plateau} is to identify the start of the plateau, note
however that we cannot guarantee that our training procedures actually
approaches a global minimum. Interestingly it was easier to observe
a plateau on MNIST rather than on the teacher network data, which
is why we had to take the minimum over 3 trials in the teacher setting.
This could be due to the change of loss (from cross entropy to the
MSE) or due to the change of the data. Note that in Figure \ref{fig:Loss-plateau} (right),
it is unclear whether the 'failed' trials , i.e. the small blue dots
with a loss above the plateau even for large widths, are stuck at
local minima of the loss or if they could have reached the plateau
if we had trained them longer.

The experiments each took between 1 and 4 hour on a single NVIDIA
GeForce GTX 1080.

\section{Equivalence for the first reformulation\label{subsec:app-first-reformulation}}
\begin{prop}[Proposition \ref{prop:first-reformulation} of the main]
The infimum of $\mathcal{L}_{\lambda}(\mathbf{W})=C(Z_{L}(X;\mathbf{W}))+\lambda\left\Vert \mathbf{W}\right\Vert ^{2},$
over the parameters $\mathrm{W}\in\mathbb{R}^{P}$ is equal to the
infimum of
\[
\mathcal{L}_{\lambda}^{r}(Z_{1},\dots,Z_{L})=C(Z_{L})+\lambda\sum_{\ell=1}^{L}\left\Vert Z_{\ell}\left(Z_{\ell-1}^{\sigma}\right)^{+}\right\Vert _{F}^{2}
\]
over the set $\mathcal{Z}$ of hidden representations $\mathbf{Z}=(Z_{\ell})_{\ell=1,\dots,L}$
such that $Z_{\ell}\in\mathbb{R}^{n_{\ell}\times N}$, $\mathrm{Im}Z_{\ell+1}^{T}\subset\mathrm{Im}\left(Z_{\ell}^{\sigma}\right)^{T}$,
with the notations $Z_{0}^{\sigma}=\left(\begin{array}{c}
X\\
\beta\mathbf{1}_{N}^{T}
\end{array}\right)$ and $Z_{\ell}^{\sigma}=\left(\begin{array}{c}
\sigma\left(Z_{\ell}\right)\\
\beta\mathbf{1}_{N}^{T}
\end{array}\right)$.

Furthermore, if $\mathbf{W}$ is a local minimizer of $\mathcal{L}_{\lambda}$
then $(Z_{1}(X; \mathbf{W}),\dots,Z_{L}(X; \mathbf{W}))$ is a local minimizer
of \textup{$\mathcal{L}_{\lambda}^{r}$.} Conversely, keeping the
same notations, if $(Z_{\ell})_{\ell=1,\ldots,L}$ is a local minimizer
of $\mathcal{L}_{\lambda}^{r}$, then $\mathrm{W}=(Z_{\ell}(Z_{\ell-1}^{\sigma})^{+})_{\ell=1,\ldots,L}$
is a local minimizer of $\mathcal{L}_{\lambda}$.
\end{prop}

\begin{proof}
We write $\Phi$ for the map which sends some weights $\mathbf{W}$
to the hidden representations $(Z_{1}(X; \mathbf{W}),\dots,Z_{L}(X; \mathbf{W}))$
and $\Psi$ for the map which sends some hidden representations $\mathbf{Z}\in\mathcal{Z}$
to $\mathbf{W}$ with weight matrices $W_{\ell}=Z_{\ell}\left(Z_{\ell-1}^{\sigma}\right)^{+}$.

We clearly have $\Phi(\Psi(\mathbf{Z}))=\mathbf{Z}$ for any $\mathbf{Z}\in\mathcal{Z}$,
however it is not true in general that $\Psi(\Phi(\mathbf{W}))$ for
all $\mathbf{W}$ (actually this is true iff $\mathbf{W}$ lies in
the image of $\Psi$).

Let $\mathcal{L}_{\lambda}(\mathbf{W})=C(Y_{\mathbf{W}})+\lambda\left\Vert \mathbf{W}\right\Vert ^{2}$
and $\mathcal{L}_{\lambda}^{r}(\mathbf{Z})=C(Z_{L})+\lambda\sum_{\ell=1}^{L}\left\Vert Z_{\ell}\left(Z_{\ell-1}^{\sigma}\right)^{+}\right\Vert _{F}^{2}$.
One can show that $\mathcal{L}_{\lambda}(\Psi(\mathbf{Z}))=\mathcal{L}_{\lambda}^r(\mathbf{Z})$
for all $\mathbf{Z}\in\mathcal{Z}$ and $\mathcal{L}_{\lambda}^r(\Phi(\mathbf{W}))\leq\mathcal{L}_{\lambda}(\mathbf{W})$
for all $\mathbf{W}$ (actually $\mathcal{L}_{\lambda}^r(\Phi(\mathbf{W}))=\mathcal{L}_{\lambda}(\mathbf{W})$
if $\mathbf{W}\in\mathrm{Im}\Psi$ and $\mathcal{L}_{\lambda}^r(\Phi(\mathbf{W}))<\mathcal{L}_{\lambda}(\mathbf{W})$
otherwise). The first fact implies that $\inf_{\mathbf{W}}\mathcal{L}_{\lambda}(\mathbf{W})\leq\inf_{\mathbf{Z}\in\mathcal{Z}}\mathcal{L}_{\lambda}^r(\mathbf{Z})$
while the second implies $\inf_{\mathbf{W}}\mathcal{L}_{\lambda}(\mathbf{W})\geq\inf_{\mathbf{Z}\in\mathcal{Z}}\mathcal{L}_{\lambda}^r(\mathbf{Z})$,
furthermore the maps $\Phi$ and $\Psi$ must map global minimizers
to global minimizers.

\textbf{Local Minima:} We now extend the correspondence
to local minima and saddles:

We prove that if $\mathbf{Z}$ is a local minimum of $\mathbf{Z}\mapsto\mathcal{L}_{\lambda}^r(\mathbf{Z})$
then $\mathbf{W}=\Psi(\mathbf{Z})$ is a local minimum of $\mathbf{W}\mapsto\mathcal{L}_{\lambda}(\mathbf{W})$
through the contrapositive: if $\mathbf{W}=\Psi(Z\mathbf{)}$ is not
a local minimum of the loss $\mathbf{W}\mapsto\mathcal{L}_{\lambda}(\mathbf{W})$
(i.e. there is a sequence of weights $\mathbf{W}_{1},\mathbf{W}_{2},\dots$
which converges to $\mathbf{W}$ with $\mathcal{L}_{\lambda}(\mathbf{W}_{i})<\mathcal{L}_{\lambda}(\mathbf{W})$
for all $i$) then $\mathbf{Z}$ is not a local minimum. We simply
consider the sequence $\mathbf{Z}_{i}=\Phi(\mathbf{W}_{i})$ which
converges to $\mathbf{Z}=\Phi(\Psi(\mathbf{Z}))$ by the continuity
of $\Phi$. This sequence satisfies $\mathcal{L}_{\lambda}^r(\mathbf{Z}_{i})\leq\mathcal{L}_{\lambda}(\mathbf{W}_{i})<\mathcal{L}_{\lambda}(\mathbf{W})=\mathcal{L}_{\lambda}^r(\mathbf{Z})$,
proving that $\mathbf{Z}$ is not a local minimum.

Let us now prove if $\mathbf{W}$ is a local minimum of $\mathbf{W}\mapsto\mathcal{L}_{\lambda}(\mathbf{W})$
then $\mathbf{Z}=\Phi(\mathbf{W})$ is a local minimum of $\mathbf{Z}\mapsto\mathcal{L}_{\lambda}^r(\mathbf{Z})$,
again using the contrapositive. Assume that there is a sequence $\mathbf{Z}_{1},\mathbf{Z}_{2},\dots$
which converges to $\mathbf{Z}=\Phi(\mathbf{W})$ with $\mathcal{L}_{\lambda}^r(\mathbf{Z}_{i})<\mathcal{L}_{\lambda}^r(\mathbf{Z})$
for all $i$. We consider the sequence $\mathbf{W}_{i}=\Psi(\mathbf{Z}_{i})$,
however this sequence might not be convergent since $\Psi$ is not
continuous, however we know the sequence is bounded, since $\left\Vert \mathbf{W}_{i}\right\Vert ^{2}=\sum_{\ell=1}^{L}\left\Vert Z_{i,\ell}\left(Z_{i,\ell-1}^{\sigma}\right)^{+}\right\Vert _{F}^{2}\leq\frac{1}{\lambda}\mathcal{L}_{\lambda}^r(\mathbf{Z}_{i})<\frac{1}{\lambda}\mathcal{L}_{\lambda}^r(\mathbf{Z})$,
this implies that there exists a subsequence $\mathbf{Z}_{k_{i}}$
such that $\mathbf{W}_{k_{i}}=\Psi\left(\mathbf{Z}_{k_{i}}\right)$
converges to some weights $\mathbf{W}'$. Note that since $\Phi(\mathbf{W})=\Phi(\mathbf{W}')$
the weight matrices must agree up to 'useless weights', i.e. for all
$\ell$
\begin{align*}
W_{\ell} & =Z_{\ell}\left(Z_{\ell-1}^{\sigma}\right)^{+}+\tilde{W}_{\ell}\\
W_{\ell}' & =Z_{\ell}\left(Z_{\ell-1}^{\sigma}\right)^{+}+\tilde{W}_{\ell}'.
\end{align*}
If $\tilde{W}_{\ell}\neq0$ then $\mathbf{W}$ is not a local minimum
(since we could choose the weights $W_{\ell}^{\epsilon}=Z_{\ell}\left(Z_{\ell-1}^{\sigma}\right)^{+}+(1-\epsilon)\tilde{W}_{\ell}$
for any $0<\epsilon<1$ to get a lower loss). We may therefore assume
$\tilde{W}_{\ell}=0$, but this implies that $\tilde{W}_{\ell}'=0$
too since $\left\Vert \mathbf{W}\right\Vert =\left\Vert \mathbf{W}'\right\Vert $
and therefore $\mathbf{W}'=\mathbf{W}$ and therefore $\mathbf{W}$
is not a local minimum since the sequence $\mathbf{W}_{k_{i}}$ approaches
$\mathbf{W}$ with a strictly lower loss.
\end{proof}

\subsection{Optimization}\label{subsec:optimization-projection}
It is possible to optimize the first reformulation directly, using projected gradient descent to guarantee that the constraints $\mathrm{Im}Z_{\ell+1}^{T}\subseteq\mathrm{Im}\left(Z_{\ell}^{\sigma}\right)^{T}$ remain satisfied. As we show now, this projection is unnecessary in the continuous case, which suggests that it might also be unnecessary in gradient descent with a small enough learning rate.

Assume there is a $\ell$ s.t. $\mathrm{Im}Z_{\ell+1}^{T}\nsubseteq\mathrm{Im}\left(Z_{\ell}^{\sigma}\right)^{T}$,
i.e. there is a vector $v\in\mathbb{R}^{N}$ (with $\left\Vert v\right\Vert =1$)
such that $v\in\ker Z_{\ell}^{\sigma}$ but $\left\Vert Z_{\ell+1}v\right\Vert >0$.
Consider any $\tilde{\mathcal{Z}}$ such that $\left\Vert \tilde{\mathcal{Z}}-\mathcal{Z}\right\Vert \leq\epsilon$,
then 
\[
\left\Vert \tilde{Z}_{\ell+1}\left(\tilde{Z}_{\ell}^{\sigma}\right)^{+}\right\Vert _{F}^{2}\geq\left\Vert \tilde{Z}_{\ell+1}vv^{T}\left(\tilde{Z}_{\ell}^{\sigma}\right)^{+}\right\Vert _{F}^{2}=\left\Vert \tilde{Z}_{\ell+1}v\right\Vert ^{2}\left\Vert v^{T}\left(\tilde{Z}_{\ell}^{\sigma}\right)^{+}\right\Vert ^{2}\geq\frac{\left\Vert Z_{\ell+1}v\right\Vert ^{2}-\epsilon}{\epsilon}.
\]
This implies that the loss explodes in the vicinity of any point where the constraints are not satisfied.
As a result, gradient flow on the cost $\mathcal{L}_{\lambda}^{r}$
starting from a value with a non-zero loss will never approach a non-acceptable
point (where $\mathrm{Im}Z_{\ell+1}^{T}\nsubseteq\mathrm{Im}\left(Z_{\ell}^{\sigma}\right)^{T}$)
since the loss is decreasing during gradient flow.

\section{Equivalence for the second reformulation\label{subsec:app-second-reformulation}}
\begin{prop}[Proposition \ref{prop:secondrepresentation} of the main]
\label{prop:correspondence-second-reformulation}For positively homogeneous
non-linearities $\sigma$, the infimum of $\mathcal{L}_{\lambda}(\mathbf{W})=C(Z_{L}(X;\mathbf{W}))+\lambda\left\Vert \mathbf{W}\right\Vert ^{2},$
over the parameters $\mathrm{W}\in\mathbb{R}^{P}$ is equal to the
infimum over $\mathcal{K}_{\mathbf{n}}(X)$ of
\[
\mathcal{L}_{\lambda}^{\mathrm{k}}(\mathbf{K},Z_{L})=C(Z_{L})+\lambda\sum_{\ell=1}^{L}\mathrm{Tr}\left[K_{\ell}\left(K_{\ell-1}^{\sigma}\right)^{+}\right].
\]
The set $\mathcal{K}_{\mathbf{n}}(X)$ is the set of covariances $\mathbf{K}=((K_{1},K_{1}^{\sigma}),\dots,(K_{L-1},K_{L-1}^{\sigma}))$
and outputs $Z_{L}$ such that for all hidden layer $\ell=1,\dots,L-1$: 
\begin{itemize}
\item the pair $\left(K_{\ell},K_{\ell}^{\sigma}\right)$ belongs to the
(translated) $n_{\ell}$-conical hull 
\[
S_{n_{\ell},\beta}=\mathrm{cone}_{n_{\ell}}\left( \left\{ \left(xx^{T},\sigma(x)\sigma(x)^{T}\right):x\in\mathbb{R}^{N}\right\} \right) +(0,\beta^{2}\mathbf{1}_{N\times N}),
\]
\item $\mathrm{Im}K_{\ell}\subset\mathrm{Im}K_{\ell-1}^{\sigma}$, with
the notation $K_{0}^{\sigma}=X^{T}X+\beta^{2}\mathbf{1}_{N\times N}$
and for the outputs, $\mathrm{Im}Z_{L}\subset\mathrm{Im}K_{L-1}^{\sigma}.$ 
\end{itemize}
\end{prop}

\begin{proof}
Consider the map $\Psi$ that maps parameters $\mathbf{W}$ to the
the tuple $(\mathbf{K},Z_{L})$. We simply need to show that the image
of $\Psi$ is the set $\mathcal{K}_{\mathbf{n}}(X)$. The fact that
$\mathrm{Im}\Psi\subset\mathcal{K}_{\mathbf{n}}(X)$ can easily be
checked.

To prove $\mathrm{Im}\Psi\supset\mathcal{K}_{\mathbf{n}}(X)$ we need
to construct a pre-image $\mathbf{W}\in\Psi^{-1}(\mathbf{K},Z_{L})$
from any tuple $\text{(}\mathbf{K},Z_{L})$ in $\mathcal{K}_{\mathbf{n}}(X)$.
For every hidden layer $\ell$, we have $(K_{\ell},K_{\ell}^{\sigma})\in S_{n_{\ell},\beta}$.
There are hence representations $Z_{\ell}\in\mathbb{R}^{n_{\ell}\times N}$
such that $K_{\ell}=Z_{\ell}^{T}Z_{\ell}$ and $K_{\ell}^{\sigma}=\sigma\left(Z_{\ell}\right)^{T}\sigma\left(Z_{\ell}\right)+\beta^{2}\mathbf{1}_{N\times N}$,
furthermore for all $\ell$, we have $\mathrm{Im}Z_{\ell}^{T}=\mathrm{Im}K_{\ell}$
and $\mathrm{Im}\left(\begin{array}{c}
\sigma\left(Z_{\ell}\right)\\
\beta\mathbf{1}_{N}^{T}
\end{array}\right)=\mathrm{Im}K_{\ell}^{\sigma}$, which implies that $\mathrm{Im}Z_{\ell}^{T}\subset\mathrm{Im}\left(Z_{\ell-1}^{\sigma}\right)^{T}$
and therefore that the tuple $\left(Z_{1},\dots Z_{L}\right)$ is
in the set $\mathcal{Z}_{\mathbf{n}}$ and we can choose the weight
matrices $W_{\ell}=Z_{\ell}\left(Z_{\ell-1}^{\sigma}\right)^{+}$
to obtain a preimage $\mathbf{W}\in\Psi^{-1}\left(\mathbf{K},Z_{L}\right)$.
\end{proof}

\subsection{Non-correspondence of the local minima}

Let us consider the map $\Gamma:\mathbf{Z}\mapsto(\mathbf{K},Z_{L})$
which maps each hidden representation $Z_{\ell}$ to the kernel pair
$(Z_{\ell}^{T}Z_{\ell},\left(Z_{\ell}^{\sigma}\right)^{T}Z_{\ell}^{\sigma})$.
The continuity of $\Gamma$ implies that if $\Gamma(\mathbf{Z})$
is a local minimum then so is $\mathbf{Z}$. The converse is not true,
instead we have:
\begin{prop}
A kernel and outputs pair $(\mathbf{K},Z_{L})$ is a local minimum
if all $\mathbf{Z}\in\Gamma^{-1}(\mathbf{K})$ are local minima. 
\end{prop}

\begin{proof}
We will prove the contrapositive of this statement: if $(\mathbf{K},Z_{L})$
is a saddle (i.e. there is a sequence $(\mathbf{K}_{1},Z_{L,1}),(\mathbf{K}_{2},Z_{L,2}),\dots$
which converges to $(\mathbf{K},Z_{L})$ such that $\mathcal{L}_{\lambda}(\mathbf{K}_{i},Z_{L,i})<\mathcal{L}_{\lambda}(\mathbf{K},Z_{L})$),
then there is a $\mathbf{Z}\in\Gamma^{-1}(\mathbf{K},Z_{L})$ which
is a saddle.

First note that for any $i$, $\Gamma^{-1}(\mathbf{K}_{i},Z_{L,i})$
is compact (it is closed and bounded since $\left\Vert Z_{\ell}\right\Vert _{F}^{2}=\mathrm{Tr}\left[K_{\ell}\right]<\infty$).
There is hence a sequence $\mathbf{Z}_{1},\mathbf{Z}_{2}.\dots$ with
$\mathbf{Z}_{i}\in\Gamma^{-1}\left(\mathbf{K}_{i},Z_{L,i}\right)$
which converges to some \textbf{$\mathbf{Z}$}. By the continuity
of $\Gamma$, we have $\Gamma(\mathbf{Z})=(\mathbf{K},Z_{L})$ and
we have $\mathcal{L}_{\lambda}\left(\mathbf{Z}_{i}\right)=\mathcal{L}_{\lambda}\left(\mathbf{K}_{i},Z_{L,i}\right)<\mathcal{L}_{\lambda}\left(\mathbf{K},Z_{L}\right)=\mathcal{L}_{\lambda}\left(\mathbf{Z}\right)$,
hence proving that $\mathbf{Z}$ is a saddle as needed.
\end{proof}
Let us now give an example of a set of weights $\mathbf{W}$ of a
depth $L=2$ network which is a local minimum of $\mathcal{L}_{\lambda}$
but such that the corresponding covariances $(\mathbf{K},Z_{L})$
are not a local minimum of $\mathcal{L}_{\lambda}^{k}$:
\begin{prop}
Consider a shallow ReLU network ($L=2$) of widths $n_{0}=1,n_{1}=2,n_{2}=1$
with no bias $\beta=0$. Consider the MSE error $\mathcal{L}_{\lambda}(\mathbf{W})=\frac{1}{N}\left\Vert Y(X;\mathbf{W})-Y\right\Vert _{F}^{2}$
for the size $N=2$ dataset with inputs $X=\left(\begin{array}{cc}
1 & -1\end{array}\right)$ and outputs $Y=\left(\begin{array}{cc}
1 & 1\end{array}\right)$.

For any $\lambda<1$ and any choices of $a_{1},a_{2}>0$ s.t. $a_{1}^{2}+a_{2}^{2}=1-\lambda$
the parameters 
\[
W_{1}=\left(\begin{array}{c}
a_{1}\\
a_{2}
\end{array}\right),W_{2}=\left(\begin{array}{cc}
a_{1} & a_{2}\end{array}\right)
\]
are a local minimum of the loss $\mathcal{L}_{\lambda}(\mathbf{W})$
however, the corresponding covariances and outputs $(K_{1},K_{1}^{\sigma}),Z_{2}$
are not a local minimum of the second reformulation $\mathcal{L}_{\lambda}^{c}((K_{1},K_{1}^{\sigma}),Z_{2})$.
\end{prop}

\begin{proof}
Consider a depth $L=2$ network with no bias ($\beta=0$) and widths
$\mathbf{n}=(1,2,1)$ with a training set of size $N=2$, with inputs
$X=(1,-1)$ and outputs $Y=(1,1)$. Let us consider this loss in the
region where all four weights are positive:
\[
W_{1}=\left(\begin{array}{c}
a_{1}\\
a_{2}
\end{array}\right),W_{2}=\left(\begin{array}{cc}
b_{1} & b_{2}\end{array}\right)
\]
with $a_{1},a_{2},b_{1},b_{2}\geq0$. We then have the following activations
\begin{align*}
Z_{1} & =\left(\begin{array}{cc}
a_{1} & -a_{1}\\
a_{2} & -a_{2}
\end{array}\right)\\
\sigma\left(Z_{1}\right) & =\left(\begin{array}{cc}
a_{1} & 0\\
a_{2} & 0
\end{array}\right)\\
Z_{2} & =\left(\begin{array}{cc}
a_{1}b_{1}+a_{2}b_{2} & 0\end{array}\right).
\end{align*}
The cost therefore takes the form
\[
\mathcal{L}_{\lambda}(\mathbf{W})=(1-a_{1}b_{1}-a_{2}b_{2})^{2}+1+\lambda\left(a_{1}^{2}+a_{2}^{2}+b_{1}^{2}+b_{2}^{2}\right).
\]
Let us now reformulate the loss in terms of the two positive values
\begin{align*}
c & =\left(\frac{a_{1}+b_{1}}{2}\right)^{2}+\left(\frac{a_{2}+b_{2}}{2}\right)^{2}\\
d & =\left(\frac{a_{1}-b_{1}}{2}\right)^{2}+\left(\frac{a_{2}-b_{2}}{2}\right)^{2}.
\end{align*}
Since $2(c+d)=a_{1}^{2}+a_{2}^{2}+b_{1}^{2}+b_{2}^{2}$ and $c-d=a_{1}b_{1}+a_{2}b_{2}$,
we can rewrite
\[
\mathcal{L}_{\lambda}(\mathbf{W})=(1-c+d)^{2}+1+2\lambda(c+d).
\]

The above is minimized (over the set of positive $c,d$) at $c=1-\lambda$
and $d=0$, since it is the unique point of the quarterplane $\left\{ \left(\begin{array}{c}
c\\
d
\end{array}\right):c,d\geq0\right\} $ where the gradient
\begin{align*}
\nabla\mathcal{L}_{\lambda}(\mathbf{W}) & =\left(\begin{array}{c}
\partial_{c}\mathcal{L}_{\lambda}(\mathbf{W})\\
\partial_{d}\mathcal{L}_{\lambda}(\mathbf{W})
\end{array}\right)=\left(\begin{array}{c}
0\\
4\lambda
\end{array}\right)
\end{align*}
points toward the inside of the quarterplane.

The set weights which optimal amongst the set of positive weights
equals the set of positive weights such that $c=1-\lambda$ and $d=0$.
Such weights $a_{1},a_{2},b_{1},b_{2}$ must satisfy $a_{1}=b_{1}$
and $a_{2}=b_{2}$ (since $d=0$) and $a_{1}^{2}+a_{2}^{2}=1-\lambda$
(since $c=1-\lambda$). In other terms, the weights of the form
\[
W_{1}=\left(\begin{array}{c}
a_{1}\\
a_{2}
\end{array}\right),W_{2}=\left(\begin{array}{cc}
a_{1} & a_{2}\end{array}\right)
\]
for any choice of positive $a_{1},a_{2}$ s.t. $a_{1}^{2}+a_{2}^{2}=1-\lambda$
(we have assumed that $\lambda<1$). For any choice of $a_{1},a_{2}$
that are both strictly positive, the above weights lie in the inside
of the set of positive weights, which implies that these weights form
a local minimum.

To prove that the corresponding covariances $(K_{1},K_{1}^{\sigma})$
are not a local minimum of the reformulation, it is sufficient to
find a pre-image of these covariances which is not a local minimum.
We will show that the extrema of the segment of local minima that
we identified are not local minima. Since all weights on the segment
have the same covariances, it follows from Proposition \ref{prop:correspondence-second-reformulation}
that if one of those points is not a local minimum, the covariances
cannot be a local minimum of the reformulation. 

Let us consider one of the extrema:
\[
W_{1}=\left(\begin{array}{c}
\sqrt{1-\lambda}\\
0
\end{array}\right),W_{2}=\left(\begin{array}{cc}
\sqrt{1-\lambda} & 0\end{array}\right).
\]
This extremum can be approached by the following weights as $\epsilon\searrow0$
\[
W_{1}^{\epsilon}=\left(\begin{array}{c}
\sqrt{1-\lambda}\\
-\epsilon
\end{array}\right),W_{2}^{\epsilon}=\left(\begin{array}{cc}
\sqrt{1-\lambda} & -\epsilon\end{array}\right).
\]

We simply need to show that for small enough $\epsilon$, we have
$\mathcal{L}_{\lambda}(\mathbf{W}^{\epsilon})<\mathcal{L}_{\lambda}(\mathbf{W})$.
Let us first compute the activations 
\begin{align*}
Z_{1} & =\left(\begin{array}{cc}
\sqrt{1-\lambda} & -\sqrt{1-\lambda}\\
-\epsilon & \epsilon
\end{array}\right)\\
\sigma\left(Z_{1}\right) & =\left(\begin{array}{cc}
\sqrt{1-\lambda} & 0\\
0 & \epsilon
\end{array}\right)\\
Z_{2} & =\left(\begin{array}{cc}
1-\lambda & \epsilon^{2}\end{array}\right).
\end{align*}
Therefore the cost $\mathcal{L}_{\lambda}(\mathbf{W}^{\epsilon})$
takes the form
\[
\mathcal{L}_{\lambda}(\mathbf{W}^{\epsilon})=(1-\lambda-1)^{2}+(\epsilon^{2}-1)^{2}+2\lambda\left((1-\lambda)+\epsilon^{2}\right).
\]
Clearly for small enough $\epsilon>0$, we have $\mathcal{L}_{\lambda}(\mathbf{W}^{\epsilon})<\mathcal{L}_{\lambda}(\mathbf{W})=\mathcal{L}_{\lambda}(\mathbf{W}^{\epsilon=0})$.
\end{proof}

\section{Description of the Plateau \label{sec:app-plateau}}
\begin{prop}[Proposition \ref{prop:rank_is_start_of_plateau} of the main]
\label{prop:reconstruction-iff-rank}Let $(\mathbf{K},Z_{L})\in\mathcal{K}(X)$,
then there are parameters $\mathbf{W}$ of a width $\mathbf{n}$ network
with covariances and outputs $\mathbf{K}$ if and only if $n_{\ell}\geq\mathrm{Rank}_{\sigma}\left(K_{\ell},K_{\ell}^{\sigma}\right)$
for all $\ell=1,\dots,L-1$.
\end{prop}

\begin{proof}
To prove that the constraints $n_{\ell}\geq\mathrm{Rank}_{\sigma}\left(K_{\ell},K_{\ell}^{\sigma}\right)$
are sufficient, we construct the parameters $\mathbf{W}$ recursively
from the first layer to the last. Since $n_{1}\geq\mathrm{Rank}_{\sigma}(K_{1},K_{1}^{\sigma})$,
there is a hidden representation $Z_{1}\in\mathbb{R}^{n_{\ell}\times N}$
such that $K_{\ell}=Z_{\ell}^{T}Z_{\ell}$ and $K_{\ell}^{\sigma}=\left(Z_{\ell}^{\sigma}\right)^{T}Z_{\ell}^{\sigma}$
(there is a representation of dimension $\mathrm{Rank}_{\sigma}(K_{1},K_{1}^{\sigma})\times N$,
but one can add some zero lines to it to obtain $Z_{1}$ without changing
the resulting $K_{\ell}$ and $K_{\ell}^{\sigma}$). Since $\mathrm{Im}Z_{1}=\mathrm{Im}K_{1}\subset\mathrm{Im}K_{0}^{\sigma}=\mathrm{Im}Z_{0}^{\sigma}$,
we can choose the parameters of the first layer as $W_{1}=Z_{1}\left(Z_{0}^{\sigma}\right)^{+}$.
All other weight matrices $W_{\ell}$ are then constructed in the
same manner.

The fact that the constraints $n_{\ell}\geq\mathrm{Rank}_{\sigma}\left(K_{\ell},K_{\ell}^{\sigma}\right)$
are necessary follows from the fact that for any network of width
$\mathbf{n}$ with parameters $\mathbf{W}$ we have that $\mathrm{Rank}_{\sigma}\left(K_{\ell}(\mathbf{W}),K_{\ell}^{\sigma}(\mathbf{W})\right)\leq n_{\ell}$
since $K_{\ell}(\mathbf{W})=\left(Z_{\ell}(\mathbf{W})\right)^{T}Z_{\ell}(\mathbf{W})$
and $K_{\ell}^{\sigma}(\mathbf{W})=\left(Z_{\ell}^{\sigma}(\mathbf{W})\right)^{T}Z_{\ell}^{\sigma}(\mathbf{W})$. 
\end{proof}

\subsection{Tightness of the upper bound}

Let us first prove the Proposition on the CP-rank of matrices resulting
from graphs without cliques:
\begin{prop}[Proposition \ref{prop:CP_rank_graph} of the main]
Given a graph $G$ with $N$ vertices and $k$ edges, consider the
$k\times N$ matrix $E$ with entries $E_{ev}=1$ if the vertex $v$
is an endpoint of the edge $e$ and $E_{ev}=0$ otherwise. The matrix
$A=E^{T}E$ is completely positive and if the graph $G$ contains
no cliques of 3 or more vertices then $\mathrm{Rank_{cp}}\left(A\right)=k$.
\end{prop}

\begin{proof}
The fact that $A=E^{T}E$ implies $\mathrm{Rank_{cp}}\left(A\right)\leq k$,
we only need to show $\mathrm{Rank_{cp}}\left(A\right)\geq k$. Let
assume that there is another decomposition $E^{T}E=B^{T}B$ for some
$m'\times N$ matrix $B$ with positive entries, we will now show
that $k'\geq k$.

First, we show that the absence of cliques of 3 or more vertices implies
that each line $B_{e}$ has at most $2$ non-zero entries. The absence
of cliques implies that for all sets $\Omega=\{v_{1},\dots,v_{r}\}$
of 3 or more vertices, there must be a pair of vertices $v,w\in\Omega$
which are not connected, i.e. $(E^{T}E)_{vw}=0$. If one line $B_{e}$
contains more than two non-zero entries, corresponding to the vertices
$\Omega=\{v_{1},\dots,v_{r}\}$ then for all $v\neq w\in\Omega$,
we have 
\[
\left(B^{T}B\right)_{vw}\geq\left(B_{e}B_{e}^{T}\right)_{vw}=1.
\]

Now if all lines $B_{e}$ have at most two non-zero entries it implies
that $B_{e}B_{e}^{T}$ has at most two non-zero off-diagonal entries.
We know that $E^{T}E$ has $2k$ non-zero off-diagonal entries. Since
\[
E^{T}E=\sum_{e=1}^{m'}B_{e}B_{e}^{T}
\]
it follows that $k'\geq k$, otherwise we could not recover all the
off-diagonal entries.
\end{proof}
We may now prove the tightness of the upper bound on the $\sigma$-rank
of the hidden representation in shallow ReLU networks without bias:
\begin{prop}[Proposition \ref{prop:example-bipartite-tight} of the main]
\label{prop:appendix-example-bipartite-tight} Consider a width-$\mathrm{n}$ shallow network ($L=2$) with ReLU activation, no bias $\beta=0$, $n_{0}=N$,
$n_{1}\geq N(N+1)$, input dataset $X_{N}=I_{N}$, and any output dataset
$Y_{N}$ such that $\left(Y_N^T Y_N \right)^{\frac{1}{2}}$ is a completely positive matrix with CP-rank $k$.

At any global minimum of $R_{\mathrm{n}}(X_{N},Y_{N})$, we have $\mathrm{Rank}_{\sigma}\left(K_{1},K_{1}^{\sigma}\right)=k$.
Furthermore for $\lambda$ small enough, at any global minimum of
$\mathcal{L}_{\lambda,\mathrm{n}}^{\mathrm{MSE}}(\mathbf{W})=\frac{1}{N}\left\Vert Y(X_{N};\mathbf{W})-Y_{N}\right\Vert _{F}^{2}+\lambda\left\Vert \mathbf{W}\right\Vert ^{2},$
we have $\mathrm{Rank}_{\sigma}\left(K_{1},K_{1}^{\sigma}\right)\geq k$.
\end{prop}

\begin{proof}
The proof is in two steps, we first show that the minimizer $\mathbf{K}$
of the representation cost has rank $k$, and then
use this to show that for small enough $\lambda$s the rank must be
at least $k$.

\textbf{Representation Cost: }We first show that at a minimizer $(K_{1},K_{1}^{\sigma})$
of the cost $\mathrm{Tr}\left[K_{1}\right]+\mathrm{Tr}\left[YY^{T}\left(K_{1}^{\sigma}\right)^{+}\right]$,
we have $K_{1}=K_{1}^{\sigma}$. This follows from the fact that if
$K_{1}\neq K_{1}^{\sigma}$, then the pair $(K_{1}^{\sigma},K_{1}^{\sigma})$
has a strictly lower cost than the pair $(K_{1},K_{1}^{\sigma})$:
for any $Z_{1}$ such that $K_{1}=Z_{1}^{T}Z_{1}$ and $K_{1}^{\sigma}=\sigma(Z_{1})^{T}\sigma(Z_{1})$,
we have that $\mathrm{Tr}\left[K_{\ell}\right]=\left\Vert Z_{1}\right\Vert _{F}^{2}\geq\left\Vert \sigma\left(Z_{1}\right)\right\Vert _{F}^{2}=\mathrm{Tr}\left[K_{\ell}^{\sigma}\right]$
and the inequality is strict if $Z_{1}\neq\sigma(Z_{1})$ (which happens
iff $K_{1}\neq K_{1}^{\sigma}$).

The optimization of the previous cost over pairs $(K_{1},K_{1}^{\sigma})$
in $S$ is therefore equivalent to the optimization of the cost $K\mapsto\mathrm{Tr}\left[K\right]+\mathrm{Tr}\left[Y^{T}YK^{+}\right]$
over completely positive matrices $K$ such that $\mathrm{Im}Y\subset\mathrm{Im}K$.
If we remove the complete positiveness constraint on $K$, then the
unique minimizer of the above is $K=\left(Y^{T}Y\right)^{\frac{1}{2}}$.
Now since $\left(Y^{T}Y\right)^{\frac{1}{2}}$
is completely positive, it is also the unique minimizer over complete
positive matrices.

We therefore have $\mathrm{Rank}_{\sigma}\left(K_{1},K_{1}^{\sigma}\right)=\mathrm{Rank}_{cp}\left(\left(Y^{T}Y\right)^{\frac{1}{2}}\right)=k$.

\textbf{Regularized Loss:} Let us consider the regularized loss 
\[
\frac{1}{N}\left\Vert Z_{2}-Y\right\Vert _{F}^{2}+\lambda\mathrm{Tr}\left[K_{1}\right]+\lambda\mathrm{Tr}\left[Z_{2}^{T}Z_{2}\left(K_{1}^{\sigma}\right)^{+}\right].
\]
The minimizer $\mathbf{K}(\lambda)=(K_{1}(\lambda),K_{1}^{\sigma}(\lambda),Z_{2}(\lambda))$
converges as $\lambda\searrow0$ to the pair $\left(K_{1},K_{1}^{\sigma},Y\right)$
where $K_{1}=K_{1}^{\sigma}$ is the minimizer of the representation
cost.

Let us now assume that there is no $\lambda_{0}$ such that for all
$\lambda<\lambda_{0}$, any minimizer $\mathbf{K}$ of the loss $\mathcal{L}_{\lambda}$
satisfies $\mathrm{Rank}_{\sigma}\left(K_{1},K_{1}^{\sigma}\right)\geq k$.
This would imply that there is a sequence $\lambda_{1},\lambda_{2},\dots$
of ridges with $\lim_{n\to\infty}\lambda_{n}=0$ and corresponding
minimizers $\mathbf{K}_{1},\mathbf{K}_{2},\dots$ (where $\mathbf{K}_{n}$
is a minimizer of the loss $\mathcal{L}_{\lambda_{n}}$) such that
$\mathrm{Rank}_{\sigma}\left(K_{n,1},K_{n,1}^{\sigma}\right) < k$.
Now by Proposition \ref{prop:reconstruction-iff-rank} for all $n$
there are parameters $\mathbf{W}_{n}$ of shallow ReLU network with
$n_{1}=k-1$ neurons in the hidden layer with covariances
equal $\mathbf{K}_{n}$. The sequence $\mathbf{W}_{1},\mathbf{W}_{2},\dots$
is uniformly bounded in norm by the representation cost $R(X_{N},Y_{N})$, there
is therefore a converging subsequence $\mathbf{W}_{n_{1}},\mathbf{W}_{n_{2}},\dots$
which converges to some parameters $\mathbf{W}$. The covariances
and outputs $\left(K_{1},K_{1}^{\sigma},Y\right)$ at these limiting
parameters $\mathbf{W}$ must minimize the representation cost, i.e.
$K_{1}=K_{1}^\sigma=\left(Y^{T}Y\right)^{\frac{1}{2}}$, but this yields a contradiction, since $\mathrm{Rank}_{\sigma}\left(K_{1},K_{1}^{\sigma}\right)=k$
but $\mathbf{W}$ are parameters of network with $n_{1}=k-1$
neurons in the hidden layer, which would imply $\mathrm{Rank}_{\sigma}\left(K_{1},K_{1}^{\sigma}\right)\leq k-1$.
\end{proof}

To show the tightness (up to constant factor) of the upper bound, one can simply apply this proposition to the special case $Y_N=E^T E$, where $E$ is the edge-vertex incidence matrix of the complete bipartite graph, in which case $k=\frac{N^2}{4}$.

We could also consider an output dataset $Y_N\in \mathbb{R}^{n_L \times} N$ whose lines are one-hot vectors, corresponding to a classification task. If we reorder the training set by class, the covariance $Y_N^T Y_N$ is a block diagonal matrix, with all ones blocks corresponding to each class. The square root $\left( Y_N^T Y_N \right)^{\frac{1}{2}}$ is also block-diagonal but the block of a class $i$ has value $\frac{1}{m_i}$ where $m_i$ is the number of datapoints in the class $i$. The matrix $\left( Y_N^T Y_N \right)^{\frac{1}{2}}$ is completely positive and has rank $k$ equal to the number of classes. This implies a much earlier plateau, which could explain why in real-world classification tasks, we observe a very early plateau.

\begin{rem}
The representation cost for $Y=E^{T}E$ is $2\left\Vert E\right\Vert _{F}^{2}=4\frac{N^{2}}{4}=N^{2}$.
We can obtain an almost optimal representation with $n_{1}=N$ neurons
by taking the weights $W_{1}=\sqrt{\frac{N}{2}}I$ and $W_{2}=\sqrt{\frac{2}{N}}E^{T}E$,
with norm $\left\Vert \sqrt{\frac{N}{2}}I\right\Vert _{F}^{2}+\left\Vert \sqrt{\frac{2}{N}}E^{T}E\right\Vert _{F}^{2}=\frac{N^{2}}{2}+\frac{2}{N}(N\frac{N^{2}}{4}+2\frac{N^{2}}{4})=\frac{N^{2}}{2}+\frac{N^{2}}{2}+N=N^{2}+N$. 
\end{rem}

\subsection{One Dimensional Shallow Network}

We now prove an upper bound on the start of the plateau for shallow networks with one-dimensional inputs and outputs:

\begin{prop}[Proposition \ref{prop:shallow_1D_plateau} of the main]
Consider shallow networks ($L=2$) with scalar inputs and outputs
($n_{0}=n_{2}=1$), a ReLU nonlinearity, and a dataset $X,Y\in\mathbb{R}^{1\times N}$.
Both the representation cost $R_{\mathbf{n}}(X,Y)$ and global minimum
$\min_{\mathbf{W}}\mathcal{L}_{\lambda,\mathbf{n}}(\mathbf{W})$ for
any $\lambda>0$ are constant as long as $n_{1}\geq4N$.
\end{prop}

\begin{proof}
We show that if there is a network with depth $L=2$ and $n_{1}>4N$
hidden neurons, we can construct a network with strictly less neurons
with the same outputs on the dataset and a smaller parameter norm.

The network function can be written in the form
\[
f_{\mathbf{W}}(x)=b+\sum_{k=1}^{n_{1}}a_{k}\sigma\left(c_{k}x+d_{k}\right).
\]
We may assume that for all neuron $i$, we have $a_{k}^{2}=c_{k}^{2}+d_{k}^{2}$
since if this is not the case, one can multiply $a_{k}$ by a scalar
and divide $c_{k}$ and $d_{k}$ by the same scalar to satisfy this
constraint while reducing the norm of the parameters.

For each neuron $i$, we define the cusp of the neuron the value $-\frac{d_{k}}{c_{k}}$,
which is the point where the neuron goes from dead to active.

If there are neurons that are inactive on the whole training set,
they can simply be removed without changing the outputs and reducing
the norm.

If there are more $4N$ neurons, we either have:
\begin{enumerate}
\item There are more than $4$ neurons whose cusp lies between two inputs
$x_{i}$ and $x_{i+1}$ (w.l.o.g. we assume $x_{1}<\dots<x_{N}$).
\item There are more than 2 neurons whose cusp lies to the left or right
of the data.
\end{enumerate}
We will now show how in the case 1, one can remove a neuron while
keeping the same outputs on the training data and reducing the norm
of the parameters. The second case is analogous.

If there are five or more neurons with a cusp between $x_{i}$ and
$x_{i+1}$, then two of those neurons $k,m$ must have the same signs
$\mathrm{sign}a_{k}=\mathrm{sign}a_{m}$ and $\mathrm{sign}c_{k}=\mathrm{sign}c_{m}$
(w.l.o.g. we assume they are all positive). We will replace these
two neurons by a single neuron $\tilde{a}\sigma(\tilde{c}x+\tilde{d})$
where $\tilde{a},\tilde{c},\tilde{d}$ are chosen as the unique positive
values ($\tilde{d}$ may be negative) to satisfy 
\begin{align*}
\tilde{a}\tilde{c} & =a_{k}c_{k}+a_{m}c_{m}\\
\tilde{a}\tilde{d} & =a_{k}d_{k}+a_{m}d_{m}\\
\tilde{a}^{2} & =\tilde{c}^{2}+\tilde{d}^{2}.
\end{align*}
First note this new neurons contributes $\tilde{a}^{2}+\tilde{c}^{2}+\tilde{d}^{2}=2\tilde{a}^{2}$
to the norm of the parameters which is less than the two previous
neurons $2a_{k}^{2}+2a_{m}^{2}$, since 
\begin{align*}
\tilde{a}^{2} & =\sqrt{\tilde{a}^{2}\left(\tilde{c}^{2}+\tilde{d}^{2}\right)}\\
 & =\sqrt{\left(a_{k}d_{k}+a_{m}d_{m}\right)^{2}+\left(a_{k}c_{k}+a_{m}c_{m}\right)^{2}}\\
 & =(a_{k}+a_{m})\sqrt{\left(\frac{a_{k}}{a_{k}+a_{m}}d_{k}+\frac{a_{m}}{a_{k}+a_{m}}d_{m}\right)^{2}+\left(\frac{a_{k}}{a_{k}+a_{m}}c_{k}+\frac{a_{m}}{a_{k}+a_{m}}c_{m}\right)^{2}}\\
 & \leq(a_{k}+a_{m})\left(\frac{a_{k}}{a_{k}+a_{m}}\sqrt{d_{k}^{2}+c_{k}^{2}}+\frac{a_{m}}{a_{k}+a_{m}}\sqrt{d_{m}^{2}+c_{m}^{2}}\right)\\
 & =a_{k}^{2}+a_{m}^{2}
\end{align*}
where the inequality follows from the convexity of the norm function
$(c,d)\mapsto\sqrt{c^{2}+d^{2}}$.

For any $x$ with $x\leq x_{i}$ or $x\geq x_{i+1}$, one can check
that
\[
\tilde{a}\sigma(\tilde{c}x+\tilde{d})=a_{k}\sigma(c_{k}x+d_{k})+a_{m}\sigma(c_{m}x+d_{m}),
\]
which implies that replacement has not changed the values of the network
on the training set.
\end{proof}

\end{document}